\documentclass{article}
\pdfoutput=1

\usepackage[accepted]{icml2021}
\usepackage{amsthm,mkolar_definitions,color,comment}
\usepackage{natbib}
\usepackage{graphicx,booktabs,subcaption}
\usepackage{algorithm,algorithmic}
\usepackage{color}
\usepackage{hyperref}

\newcount\Comments  % 0 suppresses notes to selves in text
\definecolor{darkgreen}{rgb}{0,0.5,0}
\definecolor{darkred}{rgb}{0.7,0,0}
\definecolor{teal}{rgb}{0.3,0.8,0.8}
\definecolor{orange}{rgb}{1.0,0.5,0.0}
\definecolor{purple}{rgb}{0.8,0.0,0.8}
\newcommand{\kibitz}[2]{\ifnum\Comments=1{\textcolor{#1}{\textsf{\footnotesize #2}}}\fi}

\begin{document}



\twocolumn[
\icmltitlerunning{PC-MLP}
\icmltitle{PC-MLP: Model-based Reinforcement Learning \\ with Policy Cover Guided Exploration}


\begin{icmlauthorlist}
\icmlauthor{Yuda Song}{cmu}
\icmlauthor{Wen Sun}{cornell}
\end{icmlauthorlist}

\icmlaffiliation{cornell}{Department of Computer Science, Cornell University, Ithaca , USA}
\icmlaffiliation{cmu}{Machine Learning Department, Carnegie Mellon University, Pittsburgh, USA}

\icmlcorrespondingauthor{Yuda Song}{yudas@andrew.cmu.edu}

\icmlkeywords{Machine Learning, ICML}

\vskip 0.3in
]



\printAffiliationsAndNotice{\icmlEqualContribution} % otherwise use the standard text.
\begin{abstract}
Model-based Reinforcement Learning (RL) is a popular learning paradigm due to its potential sample efficiency compared to model-free RL. However, existing empirical model-based RL approaches lack the ability to explore.  This work studies a computationally and statistically efficient model-based algorithm for both Kernelized Nonlinear Regulators (KNR) and linear Markov Decision Processes (MDPs). For both models, our algorithm guarantees polynomial sample complexity and only uses access to a planning oracle. 
%We further provide a general structural complexity measure of model-based RL, named \emph{distributional Eluder Dimension}, which captures KNRs and linear MDPs, and also generalizes the prior complexity measures such as Witness rank \citep{sun2018model}.
Experimentally, we first demonstrate the flexibility and efficacy of our algorithm on a set of exploration challenging control tasks where existing empirical model-based RL approaches completely fail. We then show that our approach retains excellent performance even in common dense reward control benchmarks that do not require heavy exploration. Finally, we demonstrate that our method can also perform reward-free exploration efficiently. {Our code can be found at \url{https://github.com/yudasong/PCMLP}.}

%and show that it significantly outperforms existing popular empirical model-based RL approaches. 

%the problem of sequential control in an unknown, nonlinear dynamical
%system, where we model the underlying system dynamics as an unknown function in a known
%Reproducing Kernel Hilbert Space. 
  %The Abstract paragraph should be indented 0.25 inch (1.5 picas) on
%  both left and right-hand margins. Use 10~point type, with a vertical
  %spacing of 11~points. The \textbf{Abstract} heading must be centered,
  %bold, and in point size 12. Two line spaces precede the
  %Abstract. The Abstract must be limited to one paragraph.
\end{abstract}

% !TeX root = main.tex

\section{Introduction}

Model-based Reinforcement Learning  (MBRL) has played a central role in Reinforcement Learning for decades and has achieved great empirical performance on tasks such as robotics \citep{deisenroth2011pilco,levine2014learning} and video games \citep{kaiser2019model}. However, most existing empirical model-based RL approaches lack the ability to perform strategic exploration. Thus they usually can not guarantee any global performance. 

\begin{figure}[t]
    \centering
    \includegraphics[width=0.22\textwidth]{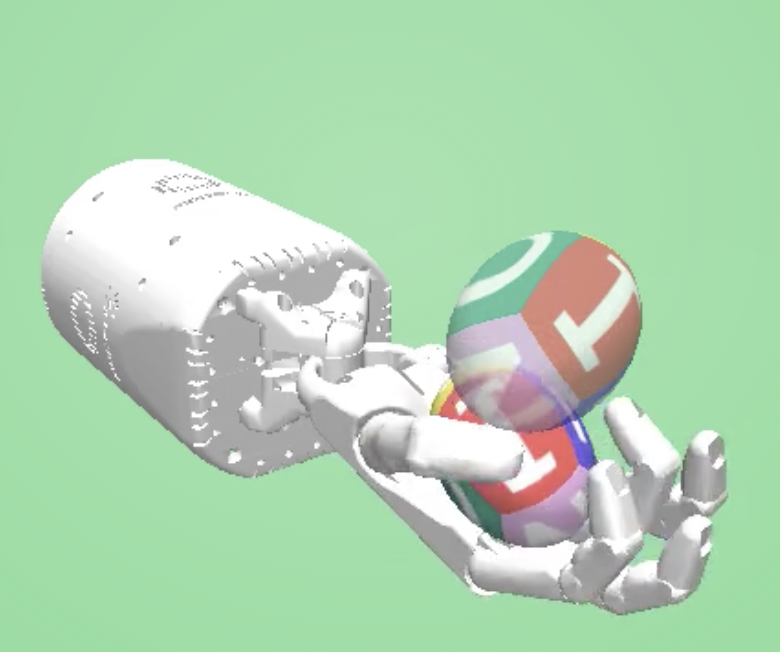}
    \includegraphics[width=0.25\textwidth]{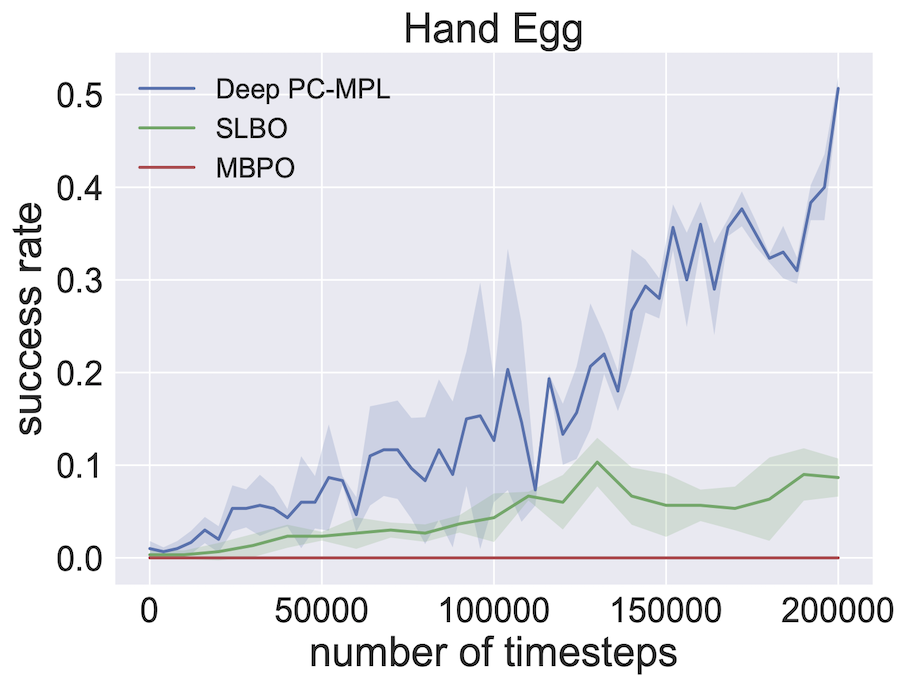}
    \centering
    \includegraphics[width=0.22\textwidth]{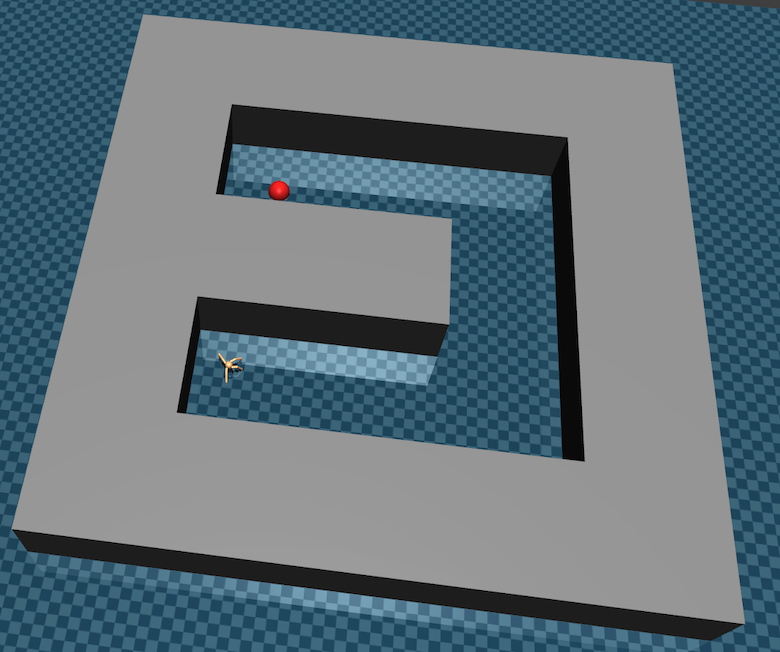}
    \includegraphics[width=0.25\textwidth]{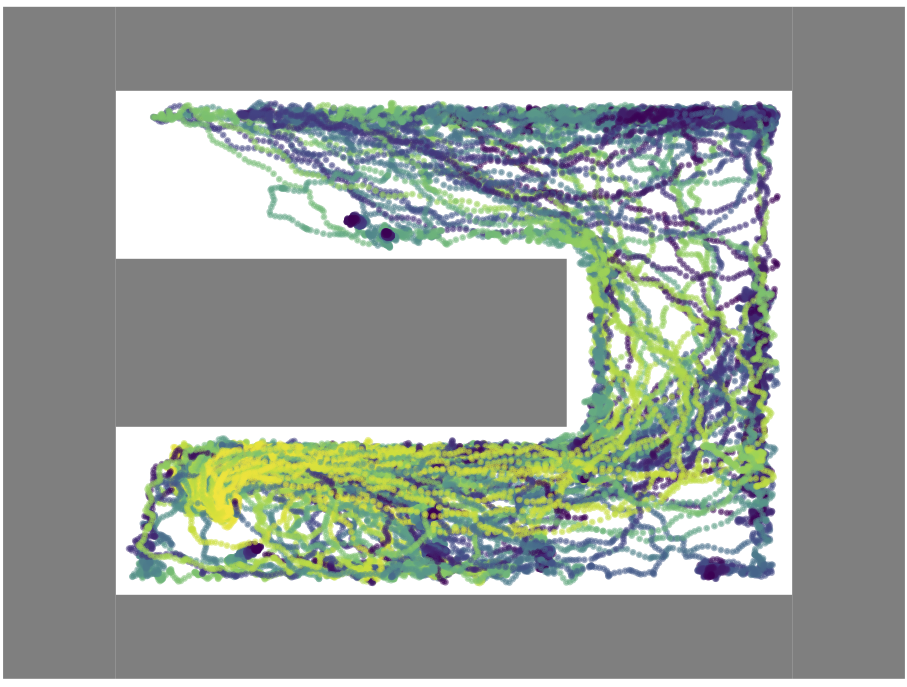}
    \caption{Example of the exploration ability of PC-MLP in the HandEgg and AntMaze experiments. The environments involve complex dynamics and sparse/no reward. Top Left: the HandEgg environment. Top Right: PC-MLP can explore strategically and thus learns much faster than other MBRL (SLBO and MBPO) baselines which rely on random exploration. Bottom Left: the AntMaze environment where we perform reward-free exploration. Bottom Right: coverage map of 5 policies in the policy cover (policy 10, 20, 30, 40 and 50). Thus PC-MLP can efficiently explore the maze without an explicit reward signal.
    %\wen{can we use the plot that contains other methods, and update the above text accordingly}
    }
    \label{fig:intro}
\end{figure}

In this work, we consider learning to control a nonlinear dynamical system. Specifically, we focus on systems that can be modeled via Reproducing Kernel Hilbert Spaces (RKHS). Following  \cite{kakade2020information}, we name such models \emph{Kernelized Nonlinear Regulators} (KNRs). Such model has been extensively used in the robotics community in the last two decades due to its flexibility to capture popular models such as linear dynamics (LQRs), piece-wise hybrid linear system,  nonlinear models that can be modeled by higher-order polynomials, and systems that can be captured by Gaussian Processes (GPs) (e.g., \citep{ko2007gaussian,deisenroth2011pilco,bansal2017goal,fisac2018general,umlauft2018uncertainty,horia_sys_id}). The fact that  KNRs have been widely used in real-world robotics and control problems proves that KNR is capable to model real-world dynamics. Thus it motivates the development of algorithms that have global performance guarantees and are also provably sample and computation efficient for KNRs.

\cite{kakade2020information} initiated an information theoretical analysis for KNRs and provided an algorithm (LC3) that achieves near-optimal regret. However, the proposed LC3 algorithm relies on an optimistic planning oracle (see Sec.~\ref{sec:prelim} for the detailed definition of an optimistic planning oracle) which is unfortunately not computationally efficient.  Thus, LC3 algorithm cannot be easily implemented using off-shelf computational tools developed from the planning and control community (e.g., highly efficient planning oracles). While its regret analysis is novel and tight, the computation inefficiency dramatically limits the practical usage of LC3.  In this work, we develop a Model-based algorithm named \emph{PC-MLP}, standing for Model Learning and Planning with Policy Cover for Exploration, that is provably sample efficient (i.e., polynomial in all relevant parameters),  and is also planning-oracle efficient, i.e., it only requires access to a classic planning oracle rather than an optimistic planning oracle. Thus PC-MLP allows one to leverage existing off-shelf efficient planning oracles from the motion planning and control community, which provide excellent flexibility when deploying the algorithm to real-world control problems.

From RL side, several new linear MDP models \citep{yang2019reinforcement,jin2019provably,modi2020sample,zhou2020provably} recently have gained a lot of interest in the theoretical RL community, although unlikely KNRs, these linear MDP models haven't been demonstrated to be applicable in real world problems. While our focus in this work is on KNRs due to their proven applicability to real-world robotics problems,  we nevertheless also analyze our algorithm directly on linear MDPs. Note that linear MDPs are different from KNRs as linear MDPs cannot capture simple linear dynamical systems such as LQR. %Specifically, we show that our algorithm is directly applicable to a specific linear MDP model proposed by \cite{yang2019reinforcement}

%The need for a computationally and sample efficient algorithm for KNRs

%Below we list the two main contributions of this paper. 
Our contributions in this work are  twofold.
%\begin{itemize}
First, theoretically, we provide a single algorithm framework PC-MLP that is provably sample efficient, and computation-wise is planning-oracle efficient (i.e., no more optimistic planning) in both KNRs and Linear MDPs \citep{yang2019reinforcement}, simultaneously. Our algorithm is modular and simple. It maintains a policy cover (i.e., an ensemble that contains all previous learned policies) and learns a model from the traces of the policies in the cover. Policy cover avoids catastrophic forgetting issues when fitting the model (i.e., during model fitting, the latest model overfits to the current policies' traces). The algorithm also uses a simple reward bonus scheme that is motivated by classic linear bandit algorithms \cite{dani2008stochastic} and also recent RL algorithms that work beyond tabular settings \cite{agarwal2020pc}. Unlike count-based reward bonus, our bonus works provably on continuous state and actions space. Empirically, we develop a practical version of PC-MLP---Deep PC-MLP, which uses deep neural network for model fitting, and random features (either Random fourier features \citep{rahimi2008random} or random network based features \citep{burda2018exploration}) for reward bonus design. We evaluate Deep PC-MLP extensively on common continuous control benchmarks including both sparse/no reward environments (e.g., sparse reward hand manipulation and high-dimensional reward-free maze shown in Fig.~\ref{fig:intro}) and dense reward environments. Our algorithm
achieves excellent performance in both exploration challenging control tasks and reward-free exploration tasks, and  the classic  dense reward control tasks. %Here we first provides an example of a challenging sparse reward environment in Fig.~\ref{fig:intro} to show how our exploration scheme helps efficiently learn the complex dynamics of the environment and thus achieves competitive performance.

This paper is structured as follows: in section 2 we provide additional related works. Section 3 introduces the KNRs and linear MDPs settings, notations, and basic assumptions. In section 4 we describe our algorithm framework, PC-MLP, while in section 5 we provide the main results on the sample complexity of PC-MLP in both linear MDPs and in KNRs. We then describe a practical implementation of PC-MLP using deep neural networks in section 6. Sections 7 includes a comprehensive empirical evaluation of the practical implementation.   
%outperforms popular MBRL algorithm in exploration challenging tasks and retains good performance in classic dense reward control tasks.
%\end{itemize}

\section{Related Works}

Below we discuss additional related works. 
On the theoretical side of KNRs, \cite{horia_sys_id} studied KNRs from a system identification perspective. Their approach relies on a reachability assumption and a Lipschitz assumption on state-action features. Our work does not require any of the two assumptions. The high-level intuition is that if there is a subspace that is not reachable, it does not matter in terms of policy optimization as no policy can reach that subspace to collect rewards. When specializing in LQR, there are a lot of prior works studying sample complexity of learning in LQRs \citep{abbasi2011regret,dean2018regret,mania2019certainty,cohen2019learning,simchowitz2020naive}. Optimal planning oracle in LQR exists and has a closed-form solution of the optimal control policy. 

On the RL side, in both theory and in practice, model-based approaches are often considered to be sample efficient \citep{deisenroth2011pilco,levine2014learning,chua2018deep,sun2018dual,kurutach2018model,nagabandi2018neural,luo2018algorithmic,ross2012agnostic,sun2018model,osband2014model,ayoub2020model,lu2019information}. Many existing empirical MBRL algorithms lack the ability to perform exploration and thus cannot automatically adapt to exploration challenging tasks. Existing theoretical works either do not apply to KNRs directly or rely on optimistic planning oracles or a Thompson sampling approach which only guarantees a regret bound in the Bayesian setting.

% !TEX root =  main.tex 

\section{PRELIMINARIES}
\label{sec:prelim}
%preliminaries

We consider episodic finite horizon MDPs $\mathcal{M} = \{ \Scal,\Acal, H, r, P^\star, s_0\}$ where $s_0$ is a fixed initial state, \footnote{Our approach generalizes to a fixed initial state distribution. We use fixed initial state to emphasize the need for exploration.} $\Scal$ and $\Acal$ are continuous state and action space, $H\in\mathbb{N}^+$ is the horizon, $r:\Scal\times\Acal\mapsto [0,1]$ is the reward function, and $P^\star:\Scal\times\Acal\mapsto \Delta(\Scal)$ is the Markovian transition. In our learning setting, we assume reward $r$ is known but transition $P^\star$ is unknown.  The learner is equipped with a policy class $\Pi \subset \Scal \mapsto \Delta(\Acal)$, and the goal is to search for an optimal policy $\pi^\star\in\Pi$, such that $\pi^\star \in \argmax_{\pi\in\Pi} J(\pi; r, P^\star)$, where $J(\pi ; r, P) := \EE\left[ \sum_{h=0}^{H-1} r(s_h,a_h) | a_h\sim \pi(\cdot|s_h), s_{h+1}\sim P(\cdot|s_h,a_h)\right]$ is the expected total reward of $\pi$ under a transition $P\in \Scal\times\Acal\mapsto\Delta(\Scal)$ and reward $r:\Scal\times\Acal\mapsto[0,1]$. In additional to policy class $\Pi$, a model-based learner is also equipped with a model class $\Pcal \subset \Scal\times\Acal\mapsto \Delta(\Scal)$.  Throughout the paper, we assume realizability in model class:
\begin{assum}[Model Realizability]
We assume $\Pcal$ is rich enough such that $P^\star \in \Pcal$.
\end{assum} 
We require the sample complexity of our learning algorithm to depend only on the complexity measures of the model class $\Pcal$.

The goal is to find an $\epsilon$-near optimal policy from $\Pi$, i.e., a policy $\hat{\pi}$ such that $J(\hat{\pi}; r,P^\star) \geq \max_{\pi\in\Pi}J(\pi; r, P^\star) - \epsilon$, with high probability, using number of samples scaling polynomially with respect to all relevant parameters including the complexity of the model class $\Pcal$. With the above setup, we discuss two specific examples.  

\paragraph{Kernelized Nonlinear Regulators} For a KNR, we denote the transition as $s' = W^\star\phi(s,a) + \epsilon$ where $\Scal\subset \mathbb{R}^{d_s}$,  $\epsilon\sim \Ncal(0,\sigma^2 I)$, $\| W^\star \|_F \leq F\in\mathbb{R}^+$, and $\phi:\Scal\times\Acal\mapsto \mathbb{R}^d $ is potentially a nonlinear mapping. 
 In other words, $P^\star(\cdot | s,a) = \Ncal( W^\star\phi(s,a), \sigma^2 I)$. Here $\phi$ is known, $\sigma$ is known, but $W^\star$ is unknown to the learner. 
When $\phi$ corresponds to some kernel feature mapping, $W^\star\phi(s,a)$ falls into the corresponding Reproducing Kernel Hilbert Space (RKHS).  In this case, the model class $\Pcal$ consists of transitions parameterized by $W$ with $\|W\|_F \leq F$. We can denote $\Wcal = \{W: \|W\|_F \leq F\}$. It is clear that $W^\star\in \Wcal$. Prior work LC3 relies on \emph{an optimistic planning oracle}, i.e., $\max_{W\in\text{Ball}}\max_{\pi\in\Pi} J(\pi; r, W)$ with $\text{Ball} \subset \Wcal$.

\paragraph{Linear MDPs} We consider a specific linear MDP model $P^\star(s' | s,a) = \langle \mu^\star(s'), \phi(s,a)\rangle,\forall s,a,s'$ with $\phi:\Scal\times\Acal\mapsto \mathbb{R}^d$, where we assume $\mu^\star \in \Upsilon \subset \Scal\mapsto \mathbb{R}^d$. Here $\mu^\star$ is unknown but $\phi$ is known. Recently \cite{agarwal2020flambe} show that such linear model has a latent representation interpretation. For this model we assume $\Upsilon$ is finite and $\mu^\star\in \Upsilon$. For norm bound, we assume $\|\mu\cdot \nu\|_2 \leq \sqrt{d}$ for any $\nu\in\mathbb{R}^{|\Scal|}$ with that $\|\nu\|_{\infty}\leq 1$.  The statistical complexity of the model class is the log of the cardinality of $\Upsilon$, i.e., $\ln\left\lvert \Upsilon \right\rvert$. This model does not directly capture the linear MDP proposed by \cite{jin2019provably}, but captures the version from \cite{yang2019reinforcement} due to its finite degree of freedom in the linear model's parameterization.\footnote{In \cite{yang2019reinforcement}, $\mu^\star(s) = M^\star\psi(s')$ where $M^\star\in\mathbb{R}^{d_{\phi}\times d_{\psi}}$ has bounded norm and $\psi$ is known to the learner. Hence standard covering argument can show that $\ln|\Upsilon|$ is equivalent to the covering dimension of the linear model parameterization.}  The algorithm from \cite{yang2019reinforcement} also relies on optimistic planning with value iteration to be the specific planning oracle.   

Note that two models, linear MDPs and KNRs, are different: one does not generalize the other.  While linear MDP generalizes the usual tabular MDPs, the Gaussian noise makes KNRs are unable to generalize tabular MDPs directly. However KNRs generalizes polynomial dynamical systems (i.e., linear system $s' = As + Ba + \epsilon$) while linear MDPs cannot.

Despite the differences in two models, we present a single model-based algorithmic framework that takes $\Pcal$ and $\Pi$ as inputs, and outputs an $\epsilon$ near-optimal policy $\pi$  in sample complexity scaling polynomially with respect to $H, 1/\epsilon, d$ and the complexity of $\Pcal$, with a polynomial number of calls to a planning oracle $\text{OP}(\Pi, r, P)$. \begin{assum}[Planning Oracle]Given reward $r$, transition $P$, and $\Pi$, we assume access to a planning oracle:
$\text{OP}( \Pi, r, P) := \argmax_{\pi\in\Pi} J(\pi; r, P)$.
\end{assum}
We treat $\text{OP}$ as a black-box planning oracle and we do not place any restrictions on the form of the oracle. It could be a  trajectory optimization based motion planner and controller \citep{ratliff2009chomp,todorov2005generalized,sun2016stochastic} or it could be an asymptotically optimal sampling-based planner \citep{williams2017model,karaman2011sampling}. Note that $\text{OP}(\Pi, r, P)$ is a computation oracle which does not use any real-world samples.

The algorithm framework simultaneously applies to both KNR and Linear MDPs, with $\Pcal = \left\{ P(\cdot|s,a): \Ncal(W\phi(s,a),\sigma^2 I ),\forall s,a \:\vert\: \|W\|_F\leq F \right\}$ and $\Pcal = \left\{P(\cdot|s,a): \mu\phi(s,a),\forall s,a \: \vert \: \mu\in\Upsilon\right\}$ as model class inputs, respectively,

\paragraph{Additional Notations} We denote $d^{\pi}_h\in\Delta(\Scal\times\Acal)$ as the state-action distribution induced by policy $\pi$ at time step $h$, and $d^{\pi} = \sum_{h=0}^{H-1} d^{\pi}_h / H$ as the average state-action distribution from $\pi$. %We denote $[n]$ as the set $\{0,\dots, n-1\}$.

%\input{eluder_generalization.tex}
% !TEX root =  main.tex 

\section{Algorithm Framework}
\begin{algorithm}[t!]
	\begin{algorithmic}[1]	
		\REQUIRE  MDP $\mathcal{M}$, inputs $(N, K, \lambda, M, \mathcal{P},c)$
		\STATE Initialize $\pi_1$
		\STATE Set policy-cover $\boldpi_1 = \{\pi_1\}$ 
        \FOR{$n = 1 \to N$}
        	   \STATE Draw $K$ samples $\{s_i,a_i\} \sim d_{\pi_n}$
	   \STATE Set $\widehat{\Sigma}_{\pi_n} = \sum_{i=1}^K \phi(s_i,a_i) \phi(s_i,a_i)^{\top} / K$ 
	   \STATE Set covariance matrix $\widehat{\Sigma}_n = \sum_{i=1}^n \Sigma_{\pi_i} + \lambda I $
            \STATE \emph{Model Learning} (MLE) with data from policy-cover $\boldpi_{n}$ and denote $\widehat{P}_n$ as an approximate optimizer of the following optimization program:
            \begin{align}
            \label{eq:mle}
            	\max_{P\in\mathcal{P}}\sum_{i=1}^M \ln P(s_i'|s_i,a_i)
		\end{align}
		where $\{s_i,a_i\} \sim d_{\boldpi_n}, s'\sim P^\star_{s_i,a_i}, \forall i\in[M]$
		\STATE Set reward bonus $\widehat{b}_n$ as in Eq.~\ref{eq:reward_bonus}%Define bonus $b_n(s,a) := c \sqrt{\phi(s,a)^{\top} \widehat\Sigma_n^{-1}\phi(s,a)}$ 
		\STATE Plan $\pi_{n+1} = \text{OP}\left( \Pi, r +  \widehat{b}_n, \widehat{P}_n,  \right)$            %\State Set absorbing MDP $\widetilde{\Mcal}_n$ based on $\widehat{\Sigma}_{n}$ and the MLE $\widehat{P}_n$ (Eq.~\ref{eq:absorbing_est})
            %\State \emph{Planning} in absorbing MDP: $\pi_{n+1} = \text{OP}\left(\widetilde{\Mcal}_n\right)$
            \STATE Update \emph{policy-cover}: $\boldpi_{n+1}  = \boldpi_n \oplus \{\pi_{n+1}\}$
            %\State Define $\rho_t := \left(\rho + (\sum_{i=0}^{n} d_n / (n+1))\right)/2$.
            %\State Collect data: $\{x_i, u_i, x'_{i}\}_{i=1}^K$ where $x_i,u_i \sim \rho_t$, $x'_i\sim P^\star_{x_i,u_i}$
            %\State Constrained Least Square: $\widehat{W}_n = \argmin_{W:\|W\|\leq M} \sum_{i=1}^K \| W\phi(x_i,u_i) - x'_{i} \|_2^2$.
            %\State Planning: $\pi_{n+1} = \argmin_{\pi\in\Pi} J(\pi; \widehat{W}_n)$
        \ENDFOR
         \end{algorithmic}
	\caption{The PC-MLP Framework}
\label{alg:dagger_2}
\end{algorithm}

In this section, we introduce our algorithmic framework. Alg.~\ref{alg:dagger_2} summarizes the algorithm \textrm{PC-MLP} standing for \textbf{M}odel \textbf{L}earning and \textbf{P}lanning using \textbf{P}olicy \textbf{C}over for Exploration. 

In high-level, the algorithm maintains a policy cover $\boldpi_n = \{\pi_1,\dots, \pi_n\}$ that contains all previously learned policies. In each episode, the model learning procedure is a maximum likelihood estimation on training data collected by $\boldpi_n$. Sampling from $\boldpi_n$ can be implemented by first sampling $i\in [1,\dots, n]$ uniformly random, and then sample $(s,a)$ from $d^{\pi_i}$. The model $\widehat{P}_n$ trained this way (via the classic Maximum Likelihood Estimation) can predict well under state-action pairs covered by $\boldpi_n$. This forces the learned model to achieve good prediction performance under the state-action pairs that are covered by the current policy cover $\boldpi_n$. 

For state-action pairs that do not well covered by $\boldpi_n$, we design reward bonus. Specifically, we form the (unnormalized and regularized) empirical covariance matrix of the policy cover $\boldpi_n$, i.e., $\widehat\Sigma_n = \sum_{i=1}^n \Sigma_{\pi_i} + \lambda I$. Intuitively, any state-action pair $(s,a)$,  whose feature $\phi(s,a)$ falls into the subspace corresponding to the eigenvectors of $\widehat\Sigma_n$ with small eigenvalues, is considered as novel. Concretely, we design reward bonus as: 
\begin{align}\label{eq:reward_bonus}
\widehat{b}_n(s,a) = \min\left\{2c \sqrt{\phi(s,a)^{\top}\widehat\Sigma_n^{-1} \phi(s,a)}, H \right\}
\end{align}  with $c$ being a parameter. 
Note we truncate it at $H$ just because we know that the total reward is always upper bounded by $H$. To gain an intuition of the reward bonus, we can think about the tabular MDP as a special case here. To model tabular MDPs, one can design $\phi(s,a)\in\mathbb{R}^{|\Scal|\Acal|}$ as a one-hot vector that encodes state-action pair $(s,a)$. In this case, $\widehat\Sigma_n$ is a diagonal matrix and a diagonal entry approximates the probability of $\boldpi_n$ visiting the corresponding state-action pair. 

With reward bonus to encourage further exploration at novel state-action pairs, we invoke the planning oracle $\textrm{OP}(\Pi, r+\widehat{b}_n, \widehat{P}_n)$ to search for the best policy from $\Pi$ that optimizes the combination $r+\widehat{b}_n$ under the empirical model $\widehat{P}_n$.  The intuition here is that the optimal policy $\pi_{n+1}$ returned by the planning oracle can visit novel state-action pairs due to the reward bonus. Thus the new policy cover $\boldpi_{n+1}$ expands the coverage of $\boldpi_n$, which leads to exploration.

Such iterative framework---iteratively fitting model and collecting new data using the optimal policy of the learned model, is widely used in practical model-based approaches (e.g., \citep{ross2012agnostic,kaiser2019model}), and our framework simply designs additional reward bonus for the planning oracle. There are prior works that leverage heuristic reward bonus inside tree search \citep{schrittwieser2019mastering} for discrete action domains, though the reward bonus in tree search in worst case performs no better than uniform random exploration \citep{munos2014bandits}.

For KNRs and Linear MDPs models, different from the LC3 algorithm \citep{kakade2020information}\footnote{While \cite{kakade2020information} experimentally rely on Thompson sampling, it is unclear if their frequentist regret bounds still hold under Thompson sampling.} and the algorithm for linear MDP from \cite{yang2019reinforcement}, our algorithm avoids the optimistic planning oracle and abstracts  the  details (e.g., optimistic value iteration) away via a black-box planning oracle. Not only optimistic planning is computationally inefficient, but it also limits the flexibility of plugging in existing efficient planning oracles developed from the motion planning and control community.
%The ability to use simply a planning oracle allows great flexibility for algorithm design. 
%While in analysis, we focus on analyzing the framework for both KNRs and linear MDPs, 
As we demonstrate in our experiments, Alg.~\ref{alg:dagger_2} has the flexibility to integrate rich function approximation (e.g., deep nets for modeling transition $P$), and furthermore state-of-the-art model-based planning algorithms. In our experiments, we use TRPO inside the learned model as the planner. Meanwhile, we also offer another implementation using MPPI as the planner, and we perform an empirical comparison between the two planners in Appendix~\ref{app:planners}. %state-of-art planning oracle from control community, e.g., Model Predictive Path Integral (MPPI) \citep{williams2017model}).
%\wen{did we try out MPPI here as well? if so, maybe we should add it here as well}

We note that the idea of policy cover was previously used in PC-PG \citep{agarwal2020pc}. While PC-PG achieves PAC bounds on linear MDPs, it cannot be applied to KNRs due to the potential nonlinear Q functions in KNRs.\footnote{ PC-PG crucially relies on the fact that in linear MDPs the quantity $\mathbb{E}_{s'\sim P^\star(\cdot|s,a)}[ f(s') ]$ is linear with respect to $\phi(s,a)$ for any function $f$. }

% !TEX root =  main.tex

\section{Analysis}
\label{sec:analysis}

%\subsection{Sample Complexity of PC-MLP}
Algorithm~\ref{alg:dagger_2} simultaneously applies to both linear MDPs and KNRs except that it takes different parameterized model class $\Pcal$ as algorithmic inputs. PC-MLP achieves polynomial sample complexity in both models, with access to a planning oracle. Note prior work on KNR requires an optimistic planning oracle, which is computationally inefficient even in linear bandit.

For KNRs, regarding the MLE oracle, we can approximately optimize it via SGD. Namely, in Eq.~\ref{eq:mle}, we can draw one sample at a time, and perform one step of SGD. We perform total $M$ steps of SGD. Due to the specific parameterization for $\mathcal{P}$ for KNRs, we can guarantee a generalization bound that the learned model parameterized by $\widehat{W}^n$, is close to $W^\star$ under $d_{\boldpi_n}$,\footnote{The $\widetilde{O}$ notation hides absolute constants and log terms.}
\begin{align*}
\mathbb{E}_{s,a\sim d_{\boldpi_n}} \| \widehat{W}^n\phi(s,a) - W^\star\phi(s,a) \|_2 \leq \widetilde{O}(1/\sqrt{M}). 
\end{align*}

\begin{theorem}[KNRs] Fix $\epsilon\in (0,H)$ and $\delta\in (0,1)$. With probability at least $1-\delta$, PC-MLP learns a policy $\pi$ such that $J(\pi; r, P^\star) \geq \max_{\pi\in \Pi} J(\pi; r, P^\star) - \epsilon$, using number of samples $\mathrm{Poly}\left( H, 1/\epsilon, \ln(1/\delta), d, d_s, F, \frac{1}{\sigma} \right)$. 
\label{thm:main_knr}
\end{theorem}
The detailed parameter setup and the polynomial dependency can be found in Theorem~\ref{thm:detailed_knr}. Unlike prior work LC3 which relies on optimistic planning, here we show that KNR is planning-oracle efficient PAC learnable. We note that in their experiments, \cite{kakade2020information} use a Thompson sampling version of LC3, though without providing a regret proof of the Thompson sampling (TS) algorithm. To the best of our knowledge, while achieving a Bayesian regret bound is possible via TS, it is unclear if TS achieves a frequentist regret bound for KNRs.

%In our analysis for KNRs, we specifically assume that we are approximating solving MLE (Eq.~\ref{eq:mle}) via Stochastic Gradient Descent (SGD). We use the SGD result to build a generalization error bound for the learned model $\widehat{P}_n$ under the distribution $d^{\boldpi_n}$ of the policy cover. Specifically, we can show that SGD returns a solution such that:
%\begin{align*}
%\mathbb{E}\left[\widehat{W}^n \right]
%\end{align*}

For linear MDPs result, we need a stronger assumption on the MLE oracle. Specifically, we are going to assume that we can exactly solve Eq.~\ref{eq:mle}, i.e., we can find the best $P\in\Pcal$ that maximizes the \emph{training} likelihood objective, i.e.,
\begin{align*}
\widehat{P}_n \in \argmax_{P\in\Pcal} \sum_{i=1}^M \ln P(s'_i | s_i,a_i).
\end{align*}
This poses a slightly stronger assumption on the computational MLE oracle. Such offline computational oracle has been widely assumed in RL works using general function approximation (e.g., \citep{ross2012agnostic,agarwal2020flambe}). Prior work \citep{agarwal2020flambe} established a generalization bound for MLE in terms of the total variation distance. Specifically, we can show that under realizability assumption $P^\star\in\Pcal$, we have the following generalization bound:
\begin{align*}
\mathbb{E}_{s,a\sim d_{\boldpi_n}} \| \widehat\mu^n \phi(s,a) - \mu^\star\phi(s,a)  \|^2_{tv} \leq \widetilde{O}(1/M).
\end{align*} See Theorem 21 from \cite{agarwal2020flambe} for example. With this MLE maximization oracle at training time, we get the following statement for linear MDPs. 
\begin{theorem}[Linear MDPs Model] Fix $\epsilon\in (0,H)$ and $\delta\in (0,1)$. With probability at least $1-\delta$, PC-MLP learns a policy $\pi$ such that $J(\pi; r, P^\star) \geq \max_{\pi\in \Pi} J(\pi; r, P^\star) - \epsilon$, using number of samples $\mathrm{Poly}\left( H, 1/\epsilon, \ln(1/\delta), d, \ln\left( |\Upsilon| \right) \right)$.
\label{thm:main_them_linear_mdp}
\end{theorem}
Note that the sample complexity scales polynomially with respect to $\ln\left( \left\lvert \Upsilon \right\rvert \right)$ rather than the cardinality $|\Upsilon|$. Note that our model generalizes the linear MDP model from \cite{yang2019reinforcement} as we can use classic covering argument here and $\ln(| \Upsilon |)$ will correspond to the covering dimension of the parameter space of the linear model from \cite{yang2019reinforcement}. We provide detailed hyper-parameter setup and polynomials in Theorem~\ref{thm:detailed_linear_mdp}, and its proofs in Appendix~\ref{app:proof}.

Theorems~\ref{thm:main_knr} and~\ref{thm:main_them_linear_mdp} indicate that the PC-MLP framework achieves an oracle-efficient PAC guarantee on KNRs and Linear MDPs simultaneously.

\section{A Practical Algorithm: Deep PC-MLP}

\begin{algorithm}[t!]
	\begin{algorithmic}[1]	
		\REQUIRE  MDP $\mathcal{M}$
		\STATE Initialize $\pi_1$
		\STATE Set replay buffer $\Dcal = \emptyset$
		\STATE Initialize model $\Pcal_{\theta}$ with parameters $\theta$ 
        \FOR{$n = 1 \dots$}
        	   \STATE Draw $K$ samples $\{s_i,a_i, s_i'\} \sim d_{\pi_n}$
	    \STATE Set $\widehat{\Sigma}_{\pi_n} = \sum_{i=1}^K \phi(s_i,a_i) \phi(s_i,a_i)^{\top} / K$ 
	   \STATE Set covariance matrix $\widehat{\Sigma}_n = \sum_{i=1}^n \Sigma_{\pi_i} + \lambda I $
	   \STATE Add $\{s_i,a_i, s_i'\}_{i=1}^K$ to replay buffer $\Dcal$
            \STATE Perform $C$ steps of  SGD on model $P_{\theta}$ (Eq.~\ref{eq:mini_batch})
             %\emph{Model Learning} (MLE) with data from policy-cover $\boldpi_{n}$ and denote $\widehat{P}_n$ as an approximate optimizer of the following optimization program:
            %\begin{align}
            %\label{eq:mle}
            	%\max_{P\in\mathcal{P}}\sum_{i=1}^M \ln P(s_i'|s_i,a_i)
		%\end{align}
		%where $\{s_i,a_i\} \sim d_{\boldpi_n}, s'\sim P^\star_{s_i,a_i}, \forall i\in[M]$
		\STATE Set reward bonus $\widehat{b}_n$ as in Eq.~\ref{eq:reward_bonus}%Define bonus $b_n(s,a) := c \sqrt{\phi(s,a)^{\top} \widehat\Sigma_n^{-1}\phi(s,a)}$ 
		\STATE  Denote $\pi_{n+1} := \text{TRPO}\left(r +  \widehat{b}_n, P_{\theta} \right)$            %\State Set absorbing MDP $\widetilde{\Mcal}_n$ based on $\widehat{\Sigma}_{n}$ and the MLE $\widehat{P}_n$ (Eq.~\ref{eq:absorbing_est})
            %\State \emph{Planning} in absorbing MDP: $\pi_{n+1} = \text{OP}\left(\widetilde{\Mcal}_n\right)$
            %\State Update \emph{policy-cover}: $\boldpi_{n+1}  = \boldpi_n \oplus \{\pi_{n+1}\}$
            %\State Define $\rho_t := \left(\rho + (\sum_{i=0}^{n} d_n / (n+1))\right)/2$.
            %\State Collect data: $\{x_i, u_i, x'_{i}\}_{i=1}^K$ where $x_i,u_i \sim \rho_t$, $x'_i\sim P^\star_{x_i,u_i}$
            %\State Constrained Least Square: $\widehat{W}_n = \argmin_{W:\|W\|\leq M} \sum_{i=1}^K \| W\phi(x_i,u_i) - x'_{i} \|_2^2$.
            %\State Planning: $\pi_{n+1} = \argmin_{\pi\in\Pi} J(\pi; \widehat{W}_n)$
        \ENDFOR
         \end{algorithmic}
	\caption{Deep PC-MLP}
\label{alg:deep_pc_mlp}
\end{algorithm}

The PC-MLP framework and its analysis from the previous sections convey three important messages: (1) use a policy cover to ensure the learned model $\widehat{P}$ is accurate at the state-action space that is covered by all previous learned policies, (2) use the bonus to reward novel state-action pairs that are not covered by all previous policies, (3) and use a planning oracle on the combined reward to balance exploration and exploitation. %In high level, if the planning oracle computes a path that visits state-action pairs with very high bonus, then we explore, i.e., we discover novel state-action pairs; on the other hand, if the planning oracle's returned path only visits state-action pairs with small 
PC-MLP relies on three modules: (1) MLE model fitting, (2) reward bonus design, (3) a planning oracle.  In this section, we instantiate these three modules which lead to a practical implementation (Alg.~\ref{alg:deep_pc_mlp}) that is used in the experiment section.

At a high level, the instantiation, \emph{Deep PC-MLP} (Alg.~\ref{alg:deep_pc_mlp}),  implements the above three modules using standard off-shelf techniques. For  bonus, it uses Eq.~\ref{eq:reward_bonus} with $\phi$ being the value of the second from the last layer of a randomly initialized neural network or a Random Fourier Feature (RFF) that corresponds to the RBF kernel \citep{rahimi2008random}. Randomly initialized network based bonus has been used in practice before in the model-free algorithm RND \cite{burda2018exploration}. Here we also demonstrate that such a bonus can be effective in the model-based setting. \footnote{In our setting, the random network shares the same structure as the deep dynamics model $P_\theta$.}

For model learning, we represent model class $\Pcal$ as a class of feedforward neural networks $P_{\theta}$ with $\theta$ being the parameters to be optimized, that takes state-action pair as the input, and outputs the predicted next state. Under the assumption that both models and the true transition have Gaussian noise with the covariance matrix $\sigma^2 I$, negative log likelihood loss simply is reduced to standard $\ell_2$ reconstruction loss. We use the standard replay buffer $\Dcal$ to store all previous $(s,a,s')$ triples. Note that since the replay buffer stores all prior experiences, it  approximates the state-action coverage from the policy cover $\boldpi_n$.
 We update the model with a few steps mini-batch SGD with the mini-batch $\{s_{i},a_{i},s'_{i}\}_{i=1}^{m+L}$ randomly sampled from the replay buffer $\Dcal$. Here we use an L-step loss given its better empirical performance demonstrated in \cite{nagabandi2018neural,luo2018algorithmic}:
 
{\begin{align}\label{eq:mini_batch}
&\theta := \nonumber\\
& \theta - \mu \sum_{i=1}^{m} \nabla_{\theta}  \left(\sum_{l=1}^L \left\|  (\hat{s}_{i+l} - \hat{s}_{i+l-1}) -( s_{i+l} - s_{i+l-1}) \right\|_2\right),
\end{align}}

where $\hat{s_i} = s_i$ and $\hat{s}_{i+l+1} = P_{\theta} ({\hat{s}_{i+l}, a_{i+l}})$. In our experiments, we use $L = 2$. \footnote{We refer the construction of the L-step loss to Eq. (6.1) in \cite{luo2018algorithmic}. Note that such update attempts to minimize the gap between the difference between consecutive predicted states and difference between consecutive ground truth states.}  
%\wen{I do not understand the above objective, why we consider the difference between between $\hat{s}_{i+\ell}$ and $\hat{s}_{i+\ell-1}$?}

For planning oracle, we use TRPO \cite{schulman2015trust} as our planner and adopt the framework in \cite{luo2018algorithmic} where in each iteration we make multiple interchangeable updates between the model and the TRPO agent. We note that we can also leverage the state-of-art model-based planner such as  Model-Predictive-Path Integral (MPPI) \citep{williams2017model} due to its excellent performance on challenging real-world continuous control tasks, for example, agile autonomous driving \citep{williams2017information}. MPPI is easy to implement and for completeness, we include the pseudocode of MPPI in Appendix~\ref{app:mppi}. We also include an empirical comparison between the two planners (TRPO and MPPI) in Appendix~\ref{app:planners}.

 %We summarize Deep PC-MLP in Alg.~\ref{alg:deep_pc_mlp}. MPPI is easy to implement and for completeness, we include the pseudocode of MPPI in Appendix~\ref{app:mppi}. We also include an emperical comparison between the two planners in Appendix~\ref{app:planners}.

% !TEX root =  main.tex

\section{Experiments}
\label{sec:exp}

% \begin{figure}[t]
%     \centering
%     \includegraphics[width=0.2\textwidth]{figures/mountaincat.jpg}
%     \includegraphics[width=0.2\textwidth]{figures/acrobot.png}
%     \caption{Visualization of our test environments with sparse reward: Left: MountainCar. Right: HandEgg.}
%     \label{fig:exp:envs}
% \end{figure}
% \wen{ Fig.~\ref{fig:exp:envs} seems not that valuable, maybe we can just delete that..}

In this section, we investigate the empirical performance of Deep PC-MPL. Our tests are mainly in three folds: we first experiment on sparse rewards experiments which are usually difficult for existing MBRL algorithms due to the demand for significant exploration. We then test our algorithm on benchmark environments with dense rewards, where exploration is not mandatory. Finally, we investigate the performance of our method in a reward-free maze exploration environment. Our results show that our practical algorithm achieves competitive results in all three scenarios and a moderate amount of exploration induced by our algorithm can even boost the performance under the dense reward settings. In this section, we refer our algorithm as the version that uses $\phi$ from the random network and TRPO as the planner. We include all experiments and hyperparameter details in Appendix~\ref{app:exp:details}.

\subsection{Sparse Reward Environments}
\label{exp:sparse}

\begin{figure}[t]
    \centering
    \includegraphics[width=0.4\textwidth]{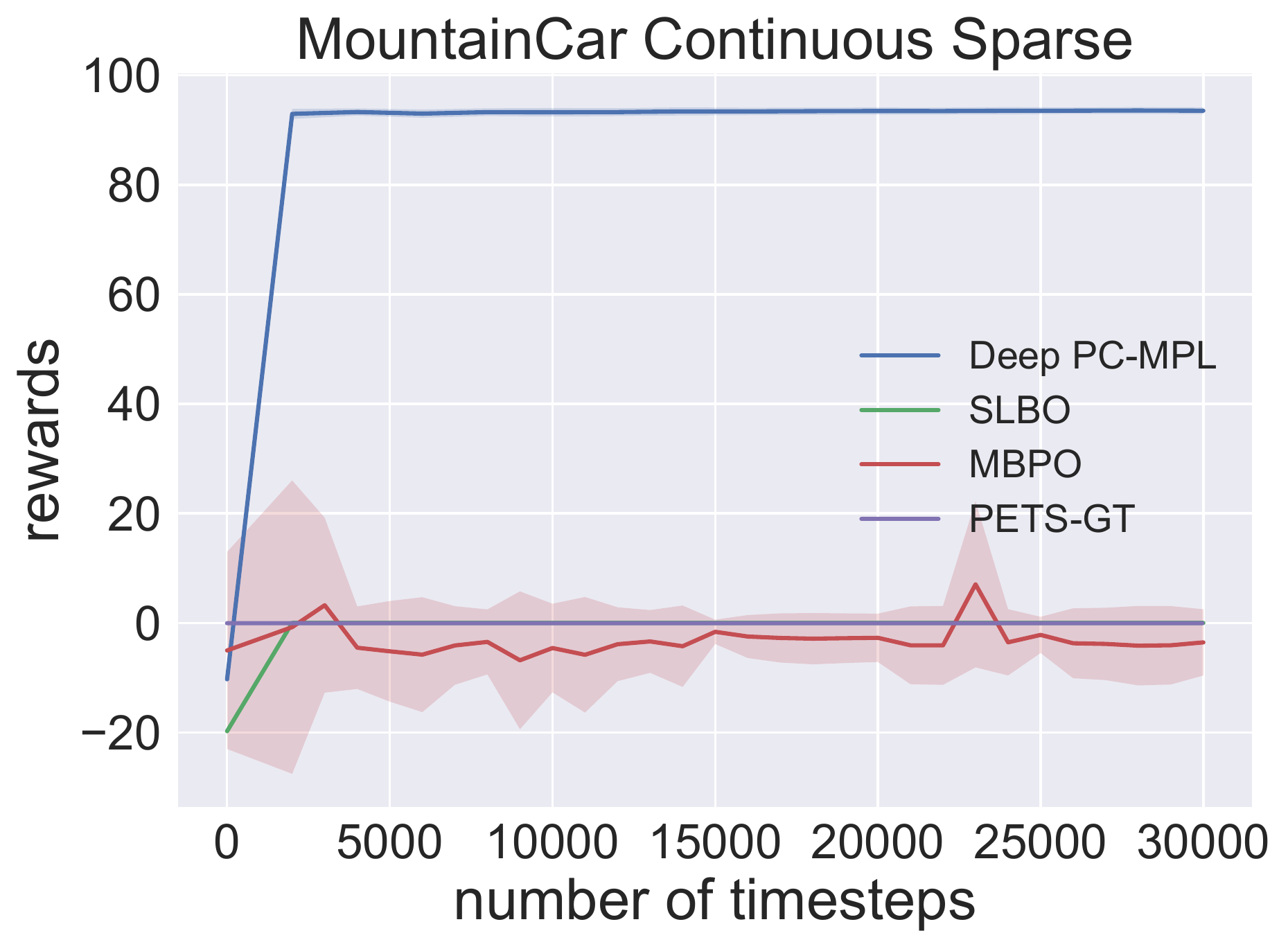}
    \centering
    \includegraphics[width=0.4\textwidth]{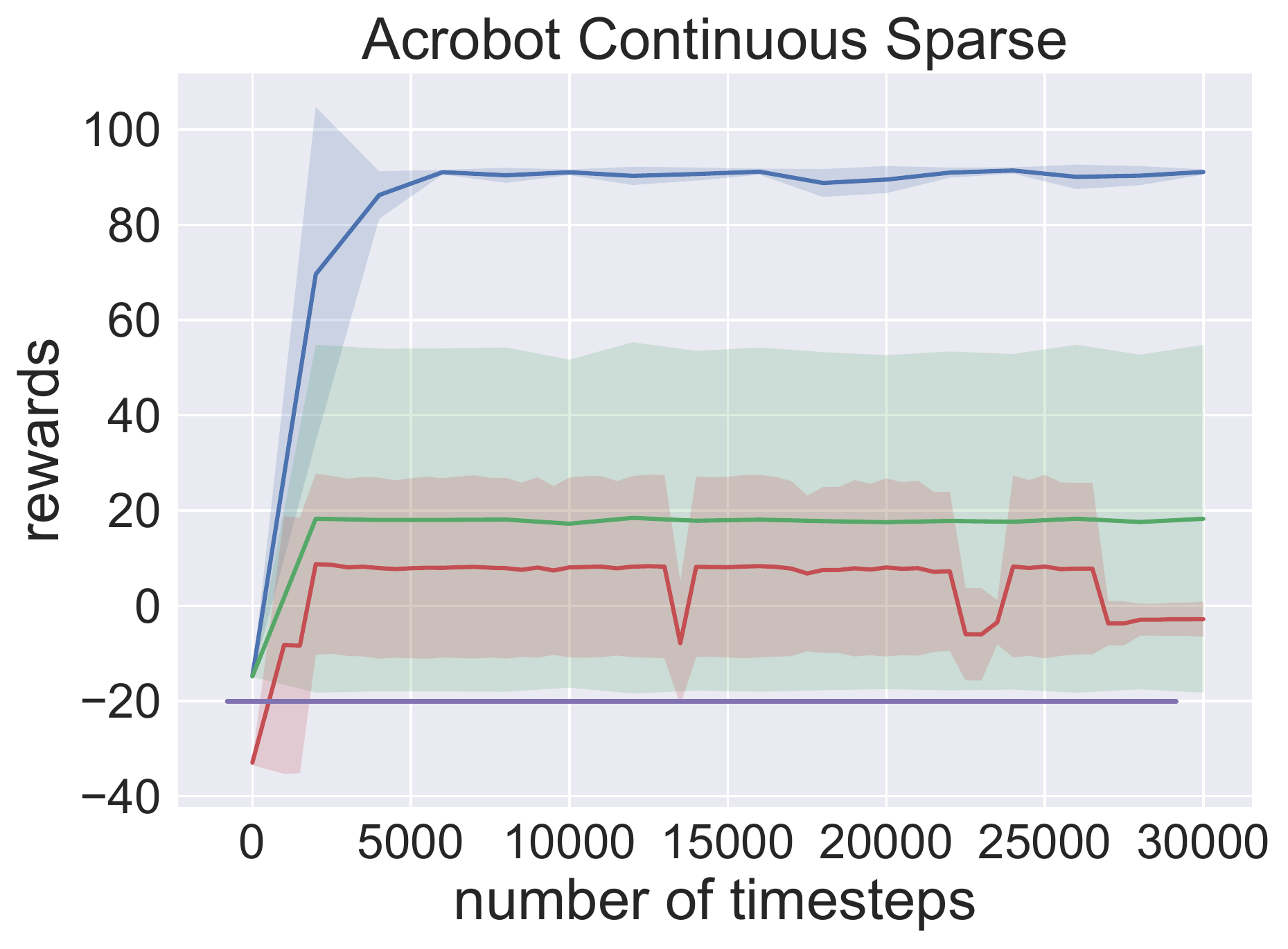}
    \caption{Episodic returns and success rates on MountainCar (top), Acrobot (bottom) environments with continuous action spaces, averaged over 4 random seeds. For the learning curve of HandEgg experiment, we refer back to Fig.~\ref{fig:intro}. The solid line denotes the mean and the shaded area denotes one standard deviation. Note in the top plot the results of SLBO and PETS-GT almost overlap after the second iteration and thus the learning curve of SLBO is covered.  %\wen{the hand figure can be just deleted now..}
    }
    \label{fig:exp:sparse}
\end{figure}

We first investigate how our algorithm mitigates the exploration issue faced by traditional MBRL algorithms with three continuous control and robotic environments in OpenAI Gym \cite{brockman2016openai}. The first environment is the Mountain Car environment \cite{Moore90} with a continuous action space of $[-1,1]$. Upon every timestep, a small control cost is incurred, and the learner receives no reward until they reach the goal where they receive a reward of 100 and the episode terminates. The second environment is Acrobot \cite{sutton1996generalization, geramifard2015rlpy}, and while the original action space is discrete, here we change the action space to continuous with range $[-1,1]$. We use the same reward scheme as in Mountain Car. The third environment is a hand manipulation task on the egg object \cite{plappert2018multi}. We follow the original task design that incurs a constant penalty for not reaching the goal but we change the reward to 10 upon reaching the goal state. We remove the rotational perturbation but still keep the positional perturbation during goal sampling. This slightly relaxes the problem since our focus is on exploration issues instead of solving goal-conditioned RL problems and one algorithm is still required to learn the rotational dynamics by strategic exploration.

We observe that these three environments are very challenging to traditional MBRL algorithms: in the first two environments, because of the existence of the motor cost, one suboptimal policy is to just take actions with minimal motor costs. Thus the model will only be accurate around the initial states and never reach the goal state. For the manipulation environment, the huge state space and complex dynamics prohibit random exploration, which could require impractical numbers of samples to accurately capture the dynamics.

Here we compare with three MBRL baselines: a) PETS-GT \cite{chua2018deep} with CEM \cite{botev2013cross}, and here we gives it the access to the \textit{ground truth} dynamics. We also include baselines with moderate exploration power: b) SLBO \cite{luo2018algorithmic} enforces entropy regularization during TRPO updates and adding Ornstein-Uhlunbeck noise while collecting samples. c) MBPO \cite{janner2019trust} uses SAC \cite{haarnoja2018soft} as the planner to encourage exploration. We plot the learning curves in Fig.~\ref{fig:exp:sparse}. In the first two environments, while all the other baselines completely fail, Deep PC-MPL achieves the optimal performance within very few model updates. This indicates that with our constructed bonus, the planner has enough exploration power to reach the goal state with small numbers of samples. The results on the manipulation task also verify our hypothesis: with strategic exploration, our algorithm captures the dynamics in a reasonable number of samples and thus the planner learns to reach the goal while planning under the accurate model. However, with the same number of samples, the other baseline (SLBO) with random exploration could barely reach the goal state with the inaccurate dynamics model due to insufficient exploration.

\begin{table}[t] 
\centering\resizebox{\columnwidth}{!}{
\begin{tabular}{|c|c|c|} 
\hline
                       &  PETS-CEM  & SLBO   \\
\hline
Mean Rank       &   5.6/11     &  4/11   \\
\hline
Median Rank       &   6/11     &  5/11   \\
\hline
&   Best MBBL & Deep PC-MLP (Ours)  \\
\hline
Mean Rank      &  4/11       &  2.4/11   \\
\hline
Median Rank      &  4/11       &  1.5/11   \\
\hline

\end{tabular}}
\caption{The rankings of our algorithm and other baselines in the 10 benchmark environments in Table~\ref{table:mujoco}. The rankings are out of 11 algorithms (10 baselines from the benchmark paper plus Deep PC-MLP). We rank the algorithms descendingly based on their mean episodic returns for each task. We then take the average position in the ranking (from 1 to 11) across the 10 tasks as the mean rank and likewise for the median rank. Note here the best mean and median rank from the benchmark paper (denoted as Best MBBL) may not be the rank of the same algorithm.}
\label{table:ranking}
\end{table}

\begin{table*}[t] 
\centering\resizebox{2\columnwidth}{!}{
\begin{tabular}{c|c|c|c|c|c} 
\toprule
                       &  HalfCheetah  & Hopper-ET & Walker2D-ET & Reacher & Ant \\
                       
\toprule
PETS-CEM               & $2795.3 \pm 879.9$            & $129.3 \pm 36.0$               & $-2.5 \pm 6.8$                & $-15.1\pm 1.7$              & $1165.5 \pm 226.9$    \\
SLBO                   & $1097.7 \pm 166.4$            & $805.7 \pm 142.4$              & $207.8 \pm 108.7$             & $-4.1 \pm 0.1$              & $718.1 \pm 123.3$     \\
Best MBBL              & \boldmath{$3639.0 \pm 1185.8$}& $926.9 \pm 154.1$              & $312.5 \pm 493.4$             & $-4.1 \pm 0.1$              & \boldmath{$1852.1 \pm 141.0$}    \\
Deep PC-MLP (Ours)     & $3599.9 \pm 662.6$            & \boldmath{$1076.1 \pm 378.6$}  & \boldmath{$1873.9 \pm 927.0$} & \boldmath{$-3.8 \pm 0.2$}   & $1583.1 \pm 903.6$    \\
TRPO                   & $-12 \pm 85.5$                & $237.4 \pm 33.5$               & $229.5 \pm 27.1$              & $-10.1 \pm 0.6$             & $323.3 \pm 24.9$      \\
\toprule
                       &  Humanoid-ET  & SlimHumanoid-ET & Swimmer & Swimmer-v0 & Pendulum \\
                       
\toprule
PETS-CEM               & $110.8 \pm 91.0$              & $355.1 \pm 157.1$              & $306.3 \pm 37.3$             & $22.1 \pm 25.2$              & $167.4 \pm 53.0$    \\
SLBO                   & $1377.0 \pm 150.4$            &  $776.1 \pm 252.5$             & $125.2 \pm 93.2$             & $46.1 \pm 18.4$              & $173.5 \pm 2.5$     \\
Best MBBL              & $1377.0 \pm 150.4$            & $1084.3 \pm 77.0$              & \boldmath{$336.3 \pm 15.8$}  & \boldmath{$85.0 \pm 98.9$}   & \boldmath{$177.3 \pm 1.9$}    \\
DEEP PC-MLP (Ours)     & \boldmath{$1775.5 \pm 322.2$} & \boldmath{$1521.0 \pm 232.6$}  & $161.8 \pm 167.8$            & $26.1 \pm 21.0$              & $174.0 \pm 2.4$    \\
TRPO                   & $289.8 \pm 5.2$               &  $281.3 \pm 10.9$              & $215.7 \pm 10.4$             & $37.9 \pm 2.0$               & $166.7 \pm 7.3$      \\
\toprule
\end{tabular}}
\caption{Final performance for benchmark Mujoco locomotion and navigation tasks. The MBRL (top 4) algorithms are evaluated after 200k real world samples and the MFRL algorithm (TRPO) is evaluated after 1 million real world samples. All the results of the baselines (both model-based and model-free ones) are directly adopted from \cite{wang2019benchmarking}. We use the same 4 random seeds as in the benchmark paper while testing our algorithm and reporting the results. Here ``ET'' in the environment name denotes that the planner has access to the termination function, which is a common assumption in current MBRL algorithms \cite{janner2019trust, rajeswaran2020game}. We use bold font to highlight the results of the algorithms with top 1 episodic rewards in each of the environments.}
\label{table:mujoco}
\end{table*}

\subsection{Dense Reward Environments}

For dense reward environments, we follow the same setup as in the MBRL benchmark paper \cite{wang2019benchmarking}. We test Deep PC-MPL in 10 Mujoco \cite{todorov2012mujoco} locomotion and navigation environments. We report the final evaluation performance of our algorithm (after training with 200k real-world samples) in Table~\ref{table:mujoco}. For reference, we include the performances of PETS-CEM \cite{chua2018deep}, SLBO \cite{luo2018algorithmic}, and the best MBRL algorithm in the benchmark paper for each corresponding environment (denoted as Best MBBL). We also include the evaluation result of our planner algorithm, TRPO, trained with 1 million real-world samples.  Following the format in \cite{wang2019benchmarking}, we also provide the ranking of each algorithm in Table~\ref{table:ranking}.

The results show that Deep PC-MPL achieves the best performances in 5 out of 10 environments, including the most challenging control environments such as Humanoid and Walker. This indicates that our exploration scheme helps boost the performances even under dense reward settings. Note that our algorithm uses the same set of hyperparameters for all 10 environments.

\subsection{Reward Free Exploration}
In this section, we investigate the performance of Deep PC-MLP in the context where there is no explicit reward signal. We test our algorithm in the high-dimensional AntMaze environment (shown in Fig.~\ref{fig:intro}), following \cite{trott2019keeping}, but here the goal is to control the ant agent to explore as many distinct states in the U-shaped maze as possible. In this task, the agent receives no explicit rewards for reaching new states, but only small rewards are given for keeping the ant agent moving, which is irrelevant to the maze exploration task (algorithms without exploration ability tend to output agents running in a circle around the spawning point). We summarize the performance of our method in Fig.~\ref{fig:reward_free}, where the results show that our method quickly covers the whole maze while random exploration completely fails to explore. We further visualize the trace map from different policies in the policy cover, which shows that different policies learn to explore different regions of the maze. The unification of the traces from the policies in the policy cover can provide a uniform coverage over the entire environment (see Fig.~\ref{fig:intro}).

\begin{figure*}[t]

    \centering
    \begin{subfigure}[t]{.24\linewidth}
    \centering\includegraphics[width=1\linewidth]{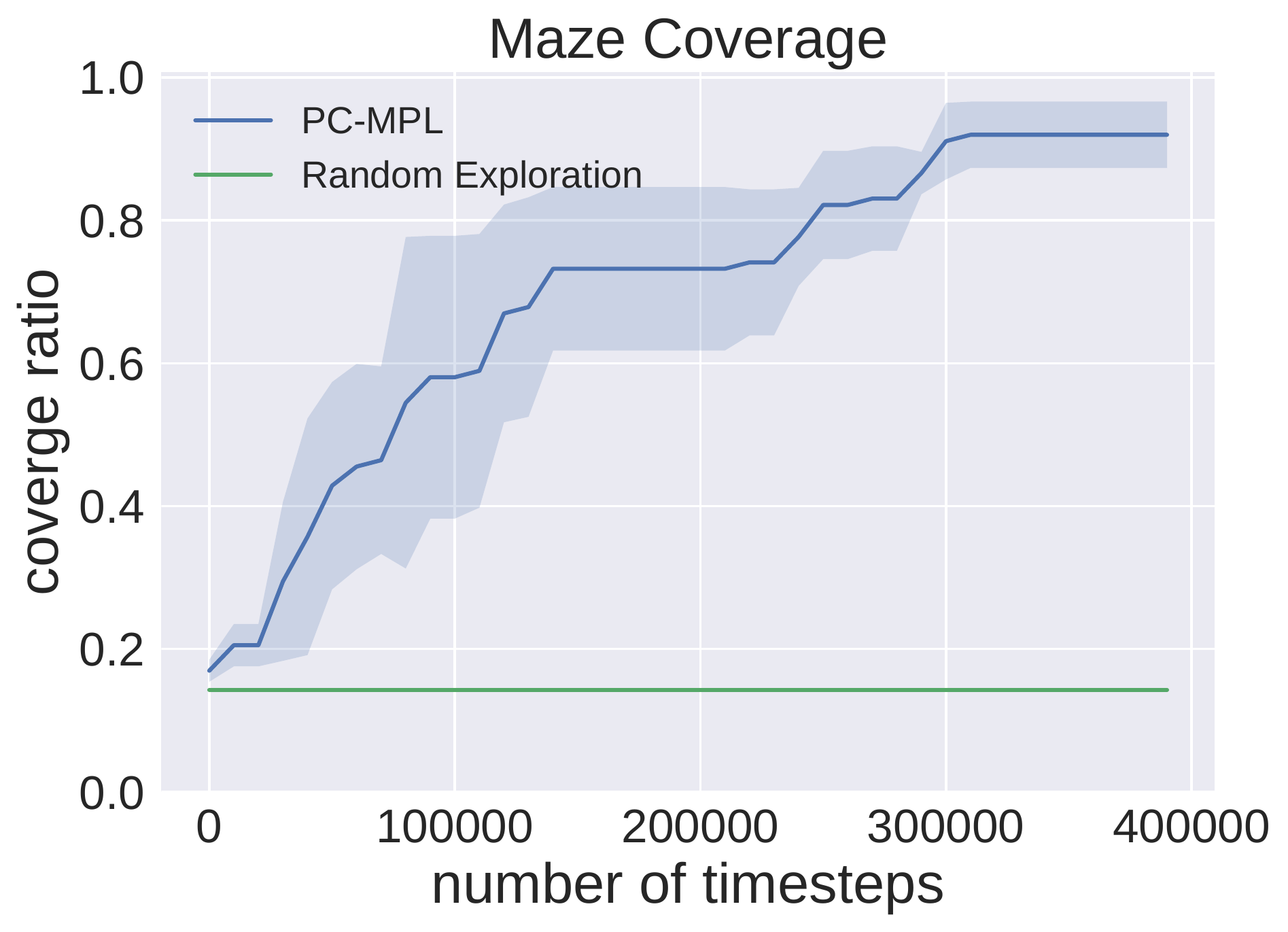}
    \caption{Coverage Curve}
   \end{subfigure}
    \begin{subfigure}[t]{.24\linewidth}
    \centering\includegraphics[width=1\linewidth]{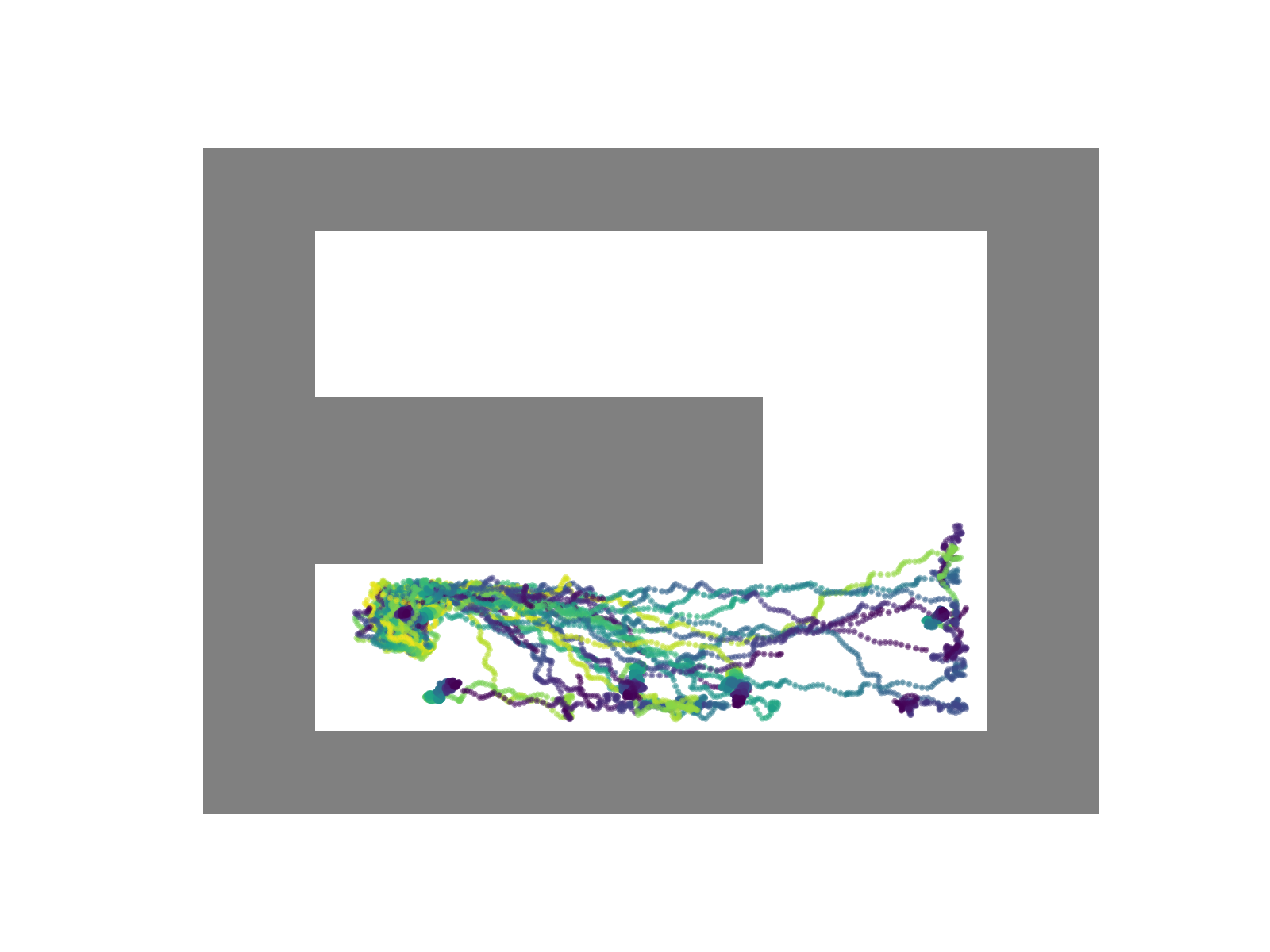}
    \caption{Policy 10}
   \end{subfigure}
    \begin{subfigure}[t]{.24\linewidth}
    \centering\includegraphics[width=1\linewidth]{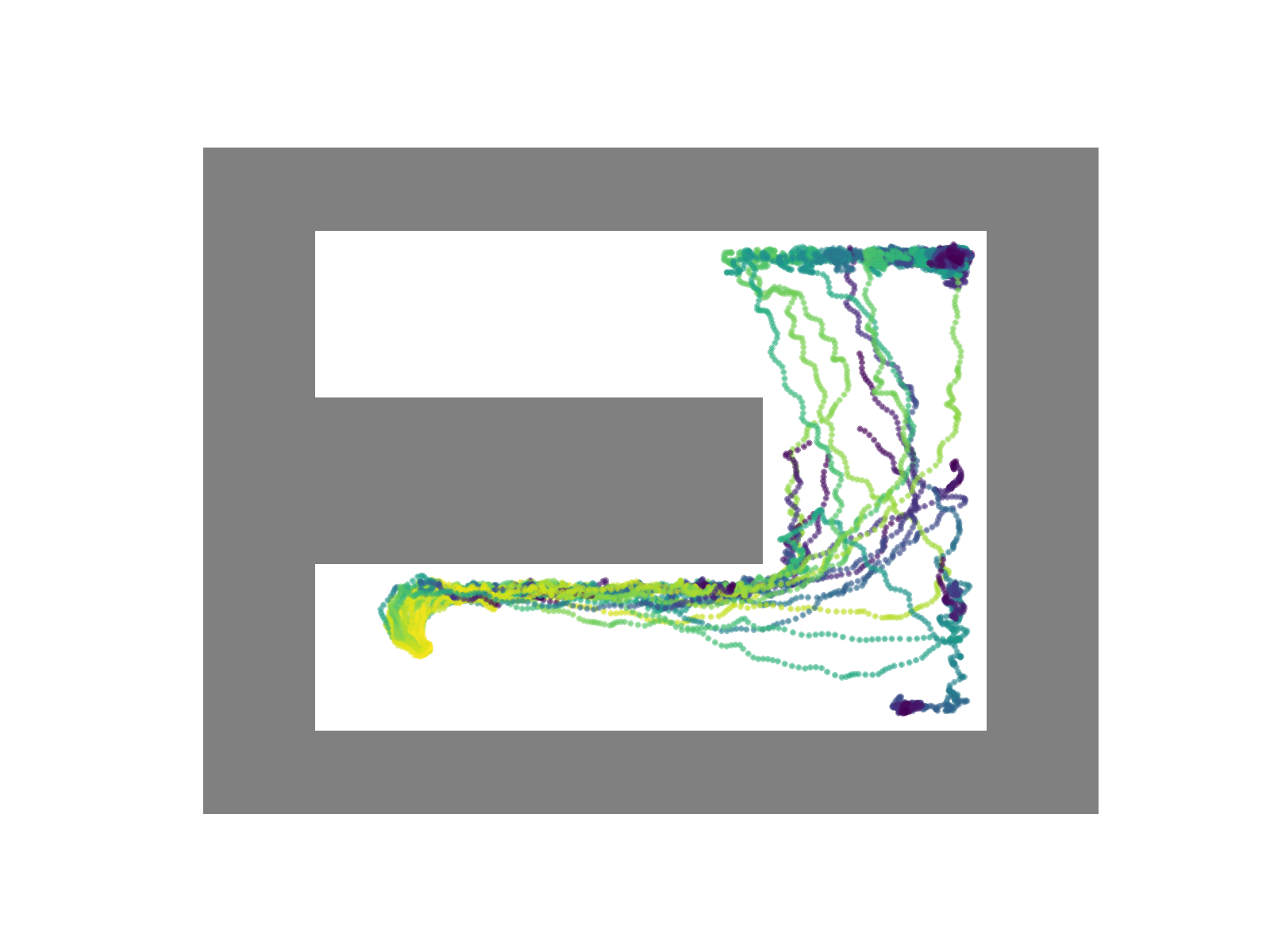}
    \caption{Policy 30}
  \end{subfigure}
  \begin{subfigure}[t]{.24\linewidth}
    \centering\includegraphics[width=1\linewidth]{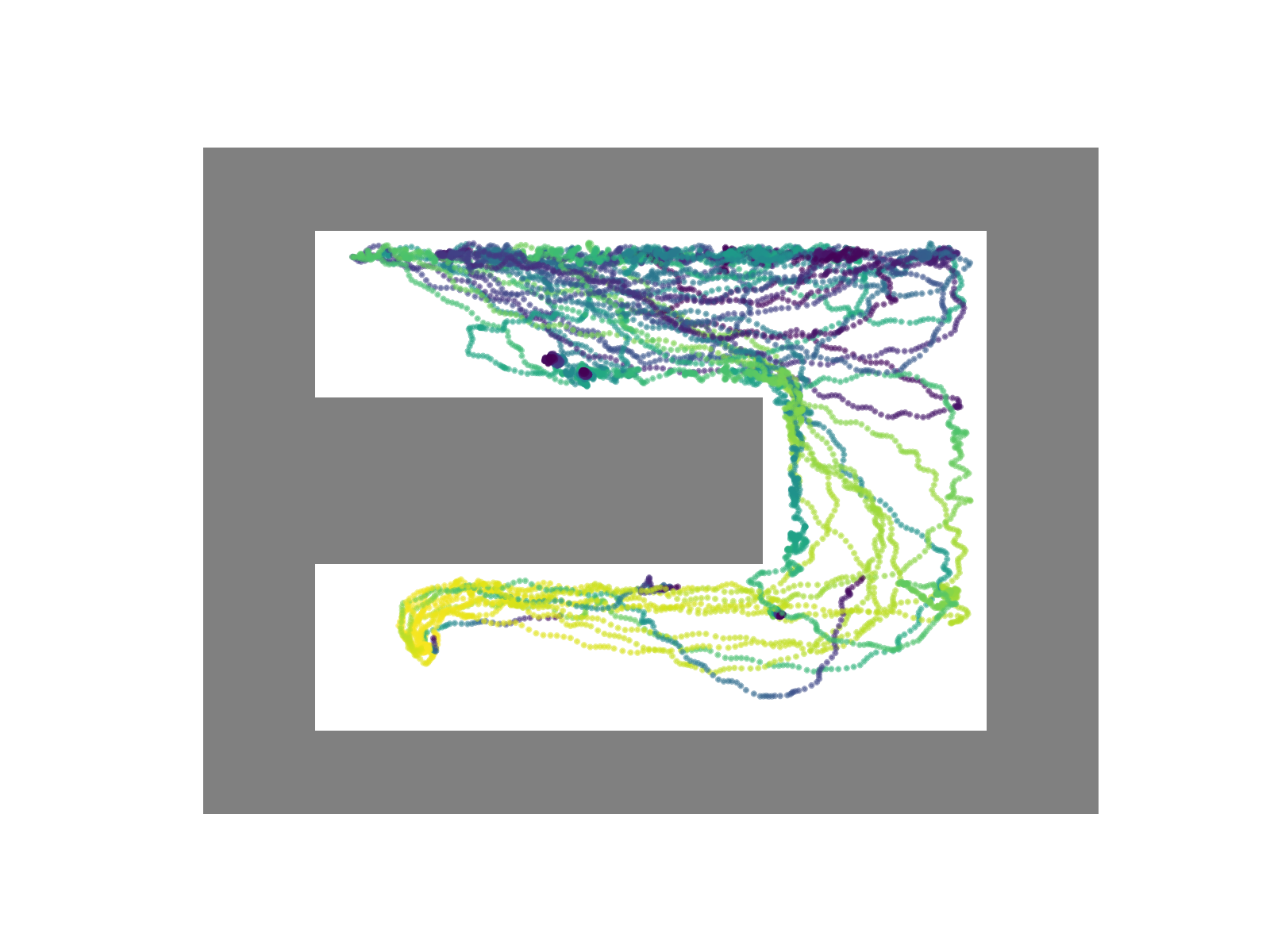}
    \caption{Policy 50}
  \end{subfigure}
  \caption{Visualization on AntMaze environment. (a): coverage curve of PC-MLP vs. random exploration. (b) Trace map from policy 10 from the policy cover. Yellow dots denote states in the earlier of the trajectories and purple dots denote later states in the trajectories. (c) Trace map from policy 30. (d) Trace map from policy 50.}
  \label{fig:reward_free}
\end{figure*}

\begin{figure}[H]
    \centering
    
    \includegraphics[width=0.4\textwidth]{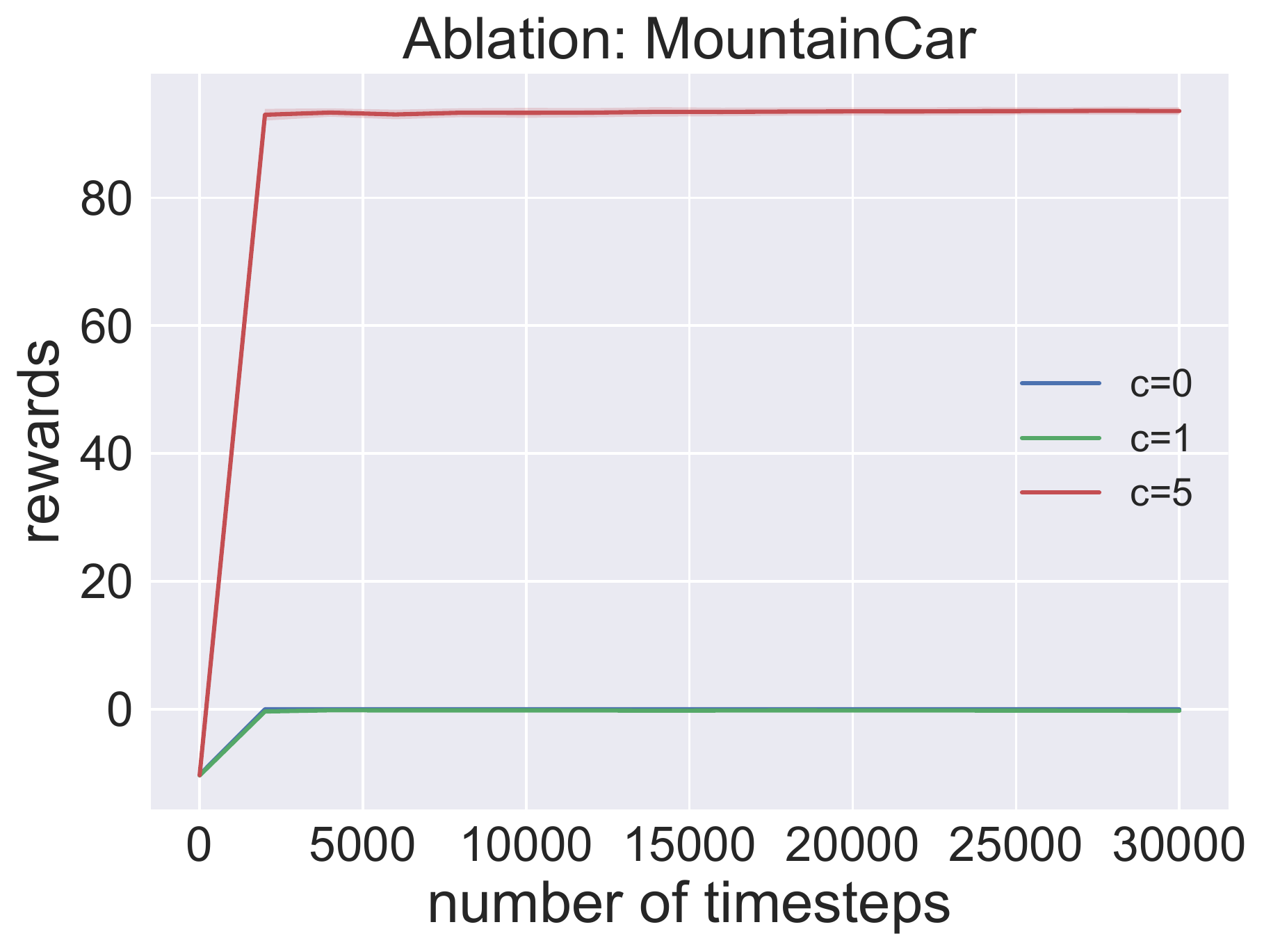}
    \includegraphics[width=0.4\textwidth]{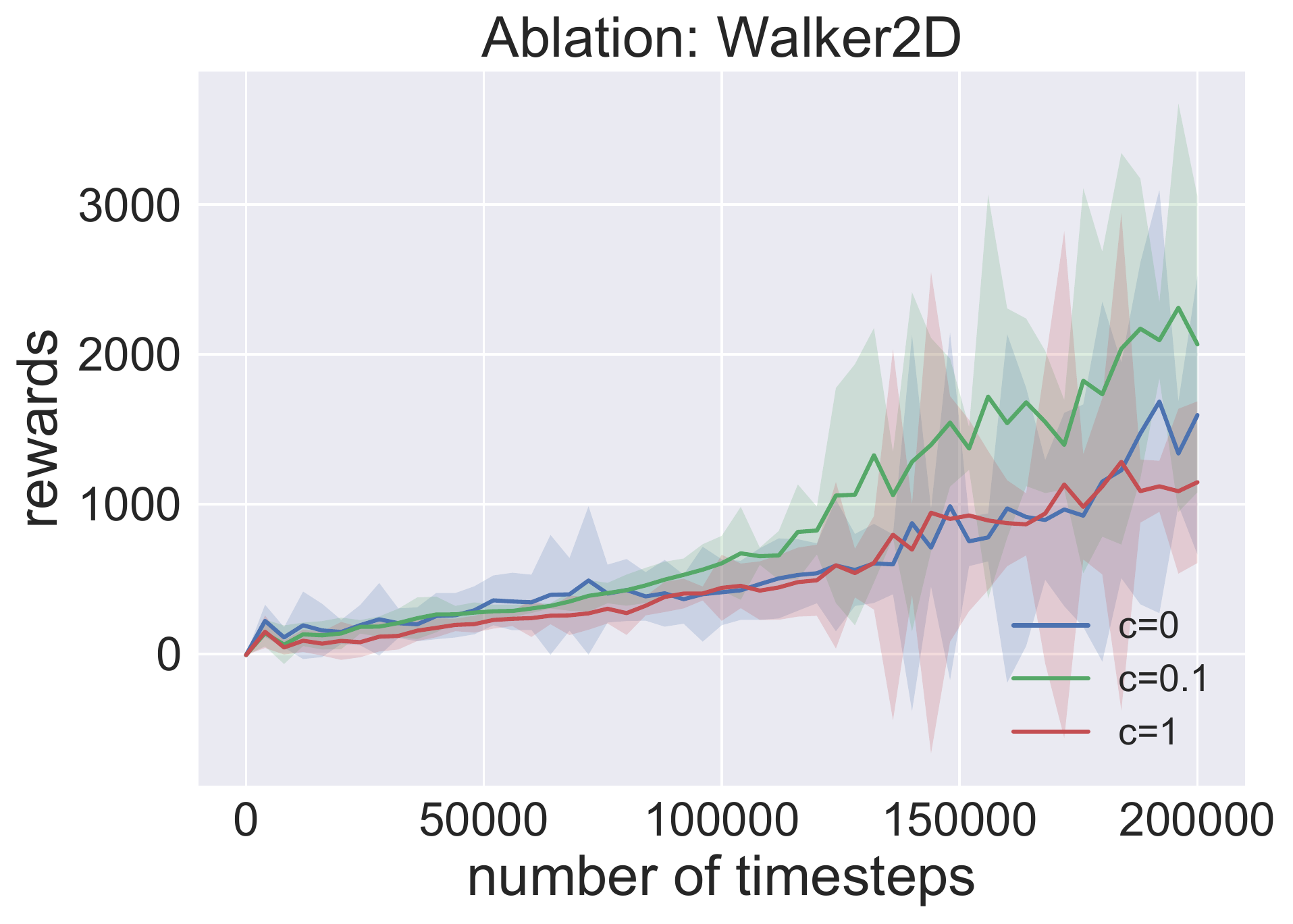}
    \caption{Ablation study: Comparison of different bonus coefficients in sparse reward environment MountainCar (top) and dense reward environment Walker2D (bottom). The results are averaged over 4 random seeds. The solid line denotes the mean and the shaded area denotes one standard deviation. In the first plot, note that the behaviors when $c=0$ and $c=1$ are identical thus the learning curves overlap.}
    \label{fig:exp:ablation}
\end{figure}

\subsection{Ablation Study on Bonus}

Next, we investigate how the magnitude of the exploration bonus affects the performance of our algorithm. We control the amount of exploration by changing the bonus coefficient $c$ defined in Eq.~\ref{eq:reward_bonus} and we show the learning curves of different coefficient choices in Fig.~\ref{fig:exp:ablation}. 

For sparse reward environments such as Mountain Car, lacking exploration results in a suboptimal policy that takes actions that minimize the motor cost, which resembles the behaviors of the other baseline algorithms in section~\ref{exp:sparse}. If the bonus signal is not strong enough ($c = 1$), we hypothesize that the control cost still dominates and thus it results in the same suboptimal policy with limited exploration around the initial states. For dense reward environments, existence of bonus ($c=0.1$) outperforms situation where there is no bonus ($c=0$). However, if the bonus signal is too strong ($c=1$), it will focus too much on exploration while ignoring exploitation. %\wen{where is $c$ defined?}
%contaminate the ground truth reward and degrade the performance. 

\subsection{Ablation Study: Vanishing Bonus}

According to the construction of bonus, our algorithm will assign small bonuses to state-action pairs that are already visited. Theoretically, the bonuses of all state-action pairs will eventually converge to 0 after we fully explore the environment. In this section, we investigate the trend of the bonus empirically in the sparse reward environment MountainCar. Fig.~\ref{fig:exp:bonus} shows the curve of the bonus per timestep that the planner receives during the evaluation phase in the real-world environment. Here all hyperparameters are fixed as in the settings in section~\ref{exp:sparse}. We discover that the average bonus converges to near 0 quickly, indicating that our algorithm can finish fully exploring the environment within the very first few iterations.

\begin{figure}[t]
    \centering
    \includegraphics[width=0.35\textwidth]{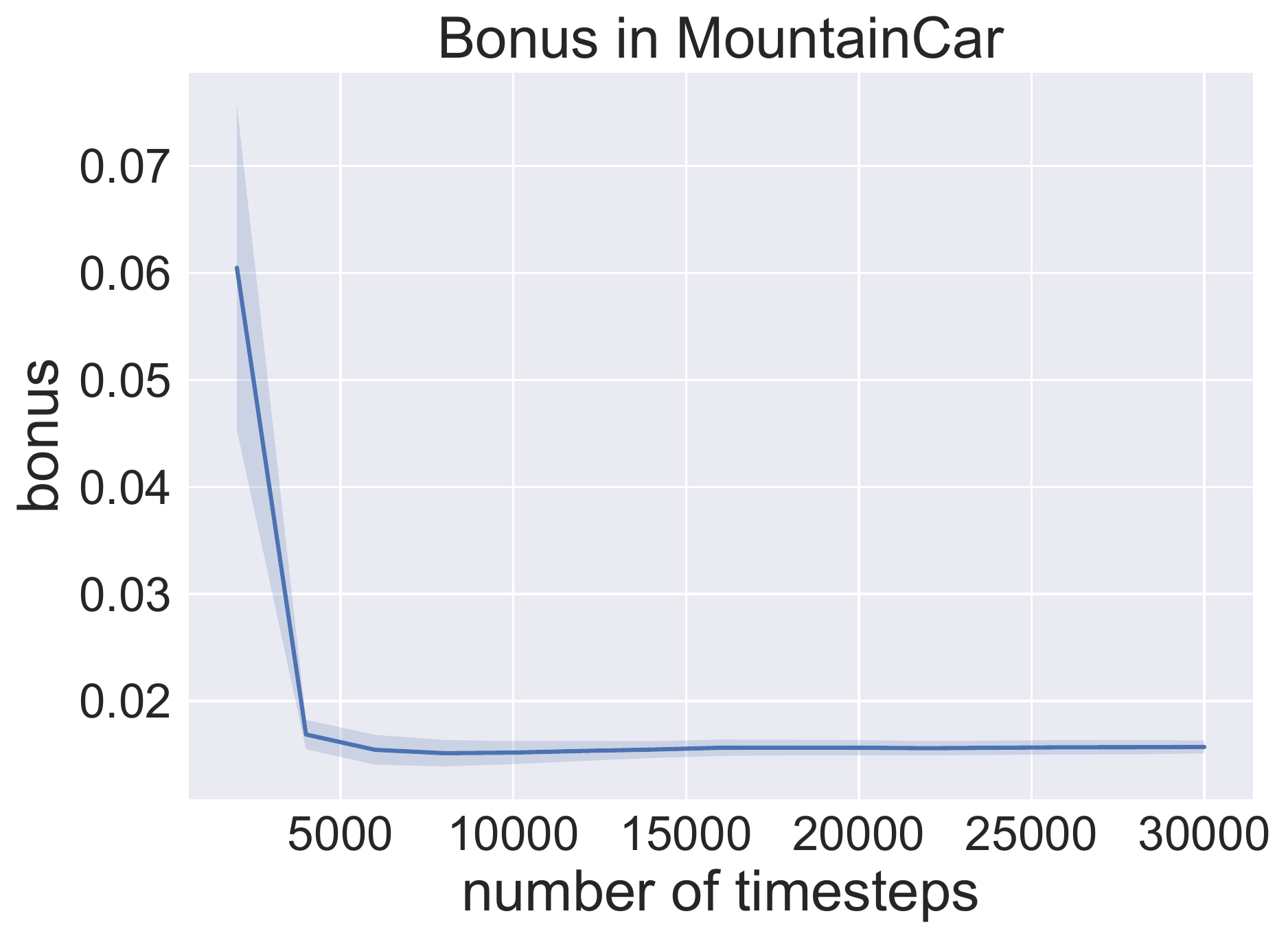}
    \caption{Decay of bonus per timestep as Deep PC-MPL fully explores the Mountain Car environment. The results are collected on the same random seeds as in section~\ref{exp:sparse}.}
    \label{fig:exp:bonus}
\end{figure}

\section{Conclusion}

In this paper, we introduce a new algorithm framework PC-MLP, which stands for Model Learning and Planning with Policy Cover for Exploration. We show that the same algorithm framework achieves polynomial sample complexity on both linear MDP and KNR models. Computation-wise, the algorithm uses a reward bonus and a black-box planner that optimizes the combination of the ground truth reward and the reward bonus inside the learned models. 

Our algorithm is modular and has great flexibility to integrate with modern rich function approximators such as deep neural networks. We provide a practical instantiation of PC-MLP where we use a deep neural network to model transition dynamics and uses reward bonus based on random features either from RFF or from a fully connected layer of a randomly initialized neural network. Our bonus scheme is simple and is motivated by classic linear bandit theory and recent RL theory on models beyond tabular MDPs. Extensive empirical results indicate that our algorithm works well on exploration challenging tasks including a high-dimensional manipulation task and reward-free exploration. For a common continuous control benchmark where the dense reward is available, our algorithm still provides competitive results. Further ablation study indicates that even for dense reward settings, a mild amount of exploration is helpful. 

%For the future works, the current results show that our algorithm is also suitable for reward-free exploration tasks. In another direction, we plan to investigate our algorithm in tasks that include both control and sparse reward exploration, such as training agents in locomotion tasks to solve a maze. When image-to-image model training becomes possible, our algorithm may also solve exploration-heavy Atari games such as Montezuma Revenge very efficiently. Finally, our results in HandEgg suggests that our algorithm could potentially solve robotics tasks such as manipulation, or can even be applied to real-world robots given the sample efficiency of our algorithm.

%\clearpage
\bibliography{refs,ref}

\begin{thebibliography}{55}
\providecommand{\natexlab}[1]{#1}
\providecommand{\url}[1]{\texttt{#1}}
\expandafter\ifx\csname urlstyle\endcsname\relax
  \providecommand{\doi}[1]{doi: #1}\else
  \providecommand{\doi}{doi: \begingroup \urlstyle{rm}\Url}\fi

\bibitem[Abbasi-Yadkori \& Szepesv{\'a}ri(2011)Abbasi-Yadkori and
  Szepesv{\'a}ri]{abbasi2011regret}
Abbasi-Yadkori, Y. and Szepesv{\'a}ri, C.
\newblock Regret bounds for the adaptive control of linear quadratic systems.
\newblock In \emph{Conference on Learning Theory}, pp.\  1--26, 2011.

\bibitem[Agarwal et~al.(2020{\natexlab{a}})Agarwal, Henaff, Kakade, and
  Sun]{agarwal2020pc}
Agarwal, A., Henaff, M., Kakade, S., and Sun, W.
\newblock Pc-pg: Policy cover directed exploration for provable policy gradient
  learning.
\newblock \emph{arXiv preprint arXiv:2007.08459}, 2020{\natexlab{a}}.

\bibitem[Agarwal et~al.(2020{\natexlab{b}})Agarwal, Kakade, Krishnamurthy, and
  Sun]{agarwal2020flambe}
Agarwal, A., Kakade, S., Krishnamurthy, A., and Sun, W.
\newblock Flambe: Structural complexity and representation learning of low rank
  mdps.
\newblock \emph{arXiv preprint arXiv:2006.10814}, 2020{\natexlab{b}}.

\bibitem[Ayoub et~al.(2020)Ayoub, Jia, Szepesvari, Wang, and
  Yang]{ayoub2020model}
Ayoub, A., Jia, Z., Szepesvari, C., Wang, M., and Yang, L.~F.
\newblock Model-based reinforcement learning with value-targeted regression.
\newblock \emph{arXiv preprint arXiv:2006.01107}, 2020.

\bibitem[Bansal et~al.(2017)Bansal, Calandra, Xiao, Levine, and
  Tomiin]{bansal2017goal}
Bansal, S., Calandra, R., Xiao, T., Levine, S., and Tomiin, C.~J.
\newblock Goal-driven dynamics learning via bayesian optimization.
\newblock In \emph{2017 IEEE 56th Annual Conference on Decision and Control
  (CDC)}, pp.\  5168--5173. IEEE, 2017.

\bibitem[Botev et~al.(2013)Botev, Kroese, Rubinstein, and
  L’Ecuyer]{botev2013cross}
Botev, Z.~I., Kroese, D.~P., Rubinstein, R.~Y., and L’Ecuyer, P.
\newblock The cross-entropy method for optimization.
\newblock In \emph{Handbook of statistics}, volume~31, pp.\  35--59. Elsevier,
  2013.

\bibitem[Brockman et~al.(2016)Brockman, Cheung, Pettersson, Schneider,
  Schulman, Tang, and Zaremba]{brockman2016openai}
Brockman, G., Cheung, V., Pettersson, L., Schneider, J., Schulman, J., Tang,
  J., and Zaremba, W.
\newblock {OpenAI} gym.
\newblock \emph{arXiv preprint arXiv:1606.01540}, 2016.

\bibitem[Burda et~al.(2019)Burda, Edwards, Storkey, and
  Klimov]{burda2018exploration}
Burda, Y., Edwards, H., Storkey, A., and Klimov, O.
\newblock Exploration by random network distillation.
\newblock In \emph{International Conference on Learning Representations}, 2019.
\newblock URL \url{https://openreview.net/forum?id=H1lJJnR5Ym}.

\bibitem[Chua et~al.(2018)Chua, Calandra, McAllister, and Levine]{chua2018deep}
Chua, K., Calandra, R., McAllister, R., and Levine, S.
\newblock Deep reinforcement learning in a handful of trials using
  probabilistic dynamics models.
\newblock In \emph{Advances in Neural Information Processing Systems}, pp.\
  4754--4765, 2018.

\bibitem[Cohen et~al.(2019)Cohen, Koren, and Mansour]{cohen2019learning}
Cohen, A., Koren, T., and Mansour, Y.
\newblock Learning linear-quadratic regulators efficiently with only sqrt{{T}}
  regret.
\newblock In \emph{International Conference on Machine Learning}, pp.\
  1300--1309, 2019.

\bibitem[Dani et~al.(2008)Dani, Hayes, and Kakade]{dani2008stochastic}
Dani, V., Hayes, T.~P., and Kakade, S.~M.
\newblock Stochastic linear optimization under bandit feedback.
\newblock In \emph{COLT}, pp.\  355--366, 2008.

\bibitem[Dean et~al.(2018)Dean, Mania, Matni, Recht, and Tu]{dean2018regret}
Dean, S., Mania, H., Matni, N., Recht, B., and Tu, S.
\newblock Regret bounds for robust adaptive control of the linear quadratic
  regulator.
\newblock In \emph{Advances in Neural Information Processing Systems}, pp.\
  4188--4197, 2018.

\bibitem[Deisenroth \& Rasmussen(2011)Deisenroth and
  Rasmussen]{deisenroth2011pilco}
Deisenroth, M. and Rasmussen, C.~E.
\newblock {PILCO}: A model-based and data-efficient approach to policy search.
\newblock In \emph{International Conference on Machine Learning}, pp.\
  465--472, 2011.

\bibitem[Devroye et~al.(2018)Devroye, Mehrabian, and Reddad]{devroye2018total}
Devroye, L., Mehrabian, A., and Reddad, T.
\newblock The total variation distance between high-dimensional gaussians.
\newblock \emph{arXiv preprint arXiv:1810.08693}, 2018.

\bibitem[Fisac et~al.(2018)Fisac, Akametalu, Zeilinger, Kaynama, Gillula, and
  Tomlin]{fisac2018general}
Fisac, J.~F., Akametalu, A.~K., Zeilinger, M.~N., Kaynama, S., Gillula, J., and
  Tomlin, C.~J.
\newblock A general safety framework for learning-based control in uncertain
  robotic systems.
\newblock \emph{IEEE Transactions on Automatic Control}, 64\penalty0
  (7):\penalty0 2737--2752, 2018.

\bibitem[Geramifard et~al.(2015)Geramifard, Dann, Klein, Dabney, and
  How]{geramifard2015rlpy}
Geramifard, A., Dann, C., Klein, R.~H., Dabney, W., and How, J.~P.
\newblock Rlpy: a value-function-based reinforcement learning framework for
  education and research.
\newblock \emph{J. Mach. Learn. Res.}, 16\penalty0 (1):\penalty0 1573--1578,
  2015.

\bibitem[Haarnoja et~al.(2018)Haarnoja, Zhou, Abbeel, and
  Levine]{haarnoja2018soft}
Haarnoja, T., Zhou, A., Abbeel, P., and Levine, S.
\newblock Soft actor-critic: Off-policy maximum entropy deep reinforcement
  learning with a stochastic actor.
\newblock In \emph{International Conference on Machine Learning}, pp.\
  1861--1870. PMLR, 2018.

\bibitem[Janner et~al.(2019)Janner, Fu, Zhang, and Levine]{janner2019trust}
Janner, M., Fu, J., Zhang, M., and Levine, S.
\newblock When to trust your model: Model-based policy optimization.
\newblock \emph{arXiv preprint arXiv:1906.08253}, 2019.

\bibitem[Jin et~al.(2019)Jin, Yang, Wang, and Jordan]{jin2019provably}
Jin, C., Yang, Z., Wang, Z., and Jordan, M.~I.
\newblock Provably efficient reinforcement learning with linear function
  approximation.
\newblock \emph{arXiv preprint arXiv:1907.05388}, 2019.

\bibitem[Kaiser et~al.(2019)Kaiser, Babaeizadeh, Milos, Osinski, Campbell,
  Czechowski, Erhan, Finn, Kozakowski, Levine, et~al.]{kaiser2019model}
Kaiser, L., Babaeizadeh, M., Milos, P., Osinski, B., Campbell, R.~H.,
  Czechowski, K., Erhan, D., Finn, C., Kozakowski, P., Levine, S., et~al.
\newblock Model-based reinforcement learning for atari.
\newblock \emph{arXiv preprint arXiv:1903.00374}, 2019.

\bibitem[Kakade et~al.(2020)Kakade, Krishnamurthy, Lowrey, Ohnishi, and
  Sun]{kakade2020information}
Kakade, S., Krishnamurthy, A., Lowrey, K., Ohnishi, M., and Sun, W.
\newblock Information theoretic regret bounds for online nonlinear control.
\newblock \emph{arXiv preprint arXiv:2006.12466}, 2020.

\bibitem[Karaman \& Frazzoli(2011)Karaman and Frazzoli]{karaman2011sampling}
Karaman, S. and Frazzoli, E.
\newblock Sampling-based algorithms for optimal motion planning.
\newblock \emph{The international journal of robotics research}, 30\penalty0
  (7):\penalty0 846--894, 2011.

\bibitem[Ko et~al.(2007)Ko, Klein, Fox, and Haehnel]{ko2007gaussian}
Ko, J., Klein, D.~J., Fox, D., and Haehnel, D.
\newblock Gaussian processes and reinforcement learning for identification and
  control of an autonomous blimp.
\newblock In \emph{Proceedings 2007 ieee international conference on robotics
  and automation}, pp.\  742--747. IEEE, 2007.

\bibitem[Kurutach et~al.(2018)Kurutach, Clavera, Duan, Tamar, and
  Abbeel]{kurutach2018model}
Kurutach, T., Clavera, I., Duan, Y., Tamar, A., and Abbeel, P.
\newblock Model-ensemble trust-region policy optimization.
\newblock \emph{arXiv preprint arXiv:1802.10592}, 2018.

\bibitem[Levine \& Abbeel(2014)Levine and Abbeel]{levine2014learning}
Levine, S. and Abbeel, P.
\newblock Learning neural network policies with guided policy search under
  unknown dynamics.
\newblock In \emph{Advances in Neural Information Processing Systems}, pp.\
  1071--1079, 2014.

\bibitem[Lu \& Van~Roy(2019)Lu and Van~Roy]{lu2019information}
Lu, X. and Van~Roy, B.
\newblock Information-theoretic confidence bounds for reinforcement learning.
\newblock In \emph{Advances in Neural Information Processing Systems}, pp.\
  2458--2466, 2019.

\bibitem[Luo et~al.(2018)Luo, Xu, Li, Tian, Darrell, and
  Ma]{luo2018algorithmic}
Luo, Y., Xu, H., Li, Y., Tian, Y., Darrell, T., and Ma, T.
\newblock Algorithmic framework for model-based deep reinforcement learning
  with theoretical guarantees.
\newblock \emph{arXiv preprint arXiv:1807.03858}, 2018.

\bibitem[Mania et~al.(2019)Mania, Tu, and Recht]{mania2019certainty}
Mania, H., Tu, S., and Recht, B.
\newblock Certainty equivalent control of {LQR} is efficient.
\newblock \emph{arXiv preprint arXiv:1902.07826}, 2019.

\bibitem[Mania et~al.(2020)Mania, Jordan, and Recht]{horia_sys_id}
Mania, H., Jordan, M.~I., and Recht, B.
\newblock Active learning for nonlinear system identification with guarantees.
\newblock \emph{arXiv preprint arXiv:2006.10277}, 2020.

\bibitem[Modi et~al.(2020)Modi, Jiang, Tewari, and Singh]{modi2020sample}
Modi, A., Jiang, N., Tewari, A., and Singh, S.
\newblock Sample complexity of reinforcement learning using linearly combined
  model ensembles.
\newblock In \emph{International Conference on Artificial Intelligence and
  Statistics}, pp.\  2010--2020, 2020.

\bibitem[Moore(1990)]{Moore90}
Moore, A.~W.
\newblock Efficient memory-based learning for robot control.
\newblock Technical report, 1990.

\bibitem[Munos(2014)]{munos2014bandits}
Munos, R.
\newblock From bandits to monte-carlo tree search: The optimistic principle
  applied to optimization and planning.
\newblock 2014.

\bibitem[Nagabandi et~al.(2018)Nagabandi, Kahn, Fearing, and
  Levine]{nagabandi2018neural}
Nagabandi, A., Kahn, G., Fearing, R.~S., and Levine, S.
\newblock Neural network dynamics for model-based deep reinforcement learning
  with model-free fine-tuning.
\newblock In \emph{IEEE International Conference on Robotics and Automation},
  pp.\  7559--7566, 2018.

\bibitem[Osband \& Van~Roy(2014)Osband and Van~Roy]{osband2014model}
Osband, I. and Van~Roy, B.
\newblock Model-based reinforcement learning and the {E}luder dimension.
\newblock In \emph{Advances in Neural Information Processing Systems}, pp.\
  1466--1474, 2014.

\bibitem[Plappert et~al.(2018)Plappert, Andrychowicz, Ray, McGrew, Baker,
  Powell, Schneider, Tobin, Chociej, Welinder, et~al.]{plappert2018multi}
Plappert, M., Andrychowicz, M., Ray, A., McGrew, B., Baker, B., Powell, G.,
  Schneider, J., Tobin, J., Chociej, M., Welinder, P., et~al.
\newblock Multi-goal reinforcement learning: Challenging robotics environments
  and request for research.
\newblock \emph{arXiv preprint arXiv:1802.09464}, 2018.

\bibitem[Rahimi \& Recht(2008)Rahimi and Recht]{rahimi2008random}
Rahimi, A. and Recht, B.
\newblock Random features for large-scale kernel machines.
\newblock In \emph{Advances in neural information processing systems}, pp.\
  1177--1184, 2008.

\bibitem[Rajeswaran et~al.(2020)Rajeswaran, Mordatch, and
  Kumar]{rajeswaran2020game}
Rajeswaran, A., Mordatch, I., and Kumar, V.
\newblock A game theoretic framework for model based reinforcement learning.
\newblock In \emph{International Conference on Machine Learning}, pp.\
  7953--7963. PMLR, 2020.

\bibitem[Ratliff et~al.(2009)Ratliff, Zucker, Bagnell, and
  Srinivasa]{ratliff2009chomp}
Ratliff, N., Zucker, M., Bagnell, J.~A., and Srinivasa, S.
\newblock Chomp: Gradient optimization techniques for efficient motion
  planning.
\newblock In \emph{2009 IEEE International Conference on Robotics and
  Automation}, pp.\  489--494. IEEE, 2009.

\bibitem[Ross \& Bagnell(2012)Ross and Bagnell]{ross2012agnostic}
Ross, S. and Bagnell, J.~A.
\newblock Agnostic system identification for model-based reinforcement
  learning.
\newblock \emph{arXiv preprint arXiv:1203.1007}, 2012.

\bibitem[Schrittwieser et~al.(2019)Schrittwieser, Antonoglou, Hubert, Simonyan,
  Sifre, Schmitt, Guez, Lockhart, Hassabis, Graepel,
  et~al.]{schrittwieser2019mastering}
Schrittwieser, J., Antonoglou, I., Hubert, T., Simonyan, K., Sifre, L.,
  Schmitt, S., Guez, A., Lockhart, E., Hassabis, D., Graepel, T., et~al.
\newblock Mastering atari, go, chess and shogi by planning with a learned
  model.
\newblock \emph{arXiv preprint arXiv:1911.08265}, 2019.

\bibitem[Schulman et~al.(2015)Schulman, Levine, Abbeel, Jordan, and
  Moritz]{schulman2015trust}
Schulman, J., Levine, S., Abbeel, P., Jordan, M., and Moritz, P.
\newblock Trust region policy optimization.
\newblock In \emph{International Conference on Machine Learning}, pp.\
  1889--1897, 2015.

\bibitem[Simchowitz \& Foster(2020)Simchowitz and Foster]{simchowitz2020naive}
Simchowitz, M. and Foster, D.~J.
\newblock Naive exploration is optimal for online {LQR}.
\newblock \emph{arXiv preprint arXiv:2001.09576}, 2020.

\bibitem[Sun et~al.(2016)Sun, van~den Berg, and Alterovitz]{sun2016stochastic}
Sun, W., van~den Berg, J., and Alterovitz, R.
\newblock Stochastic extended lqr for optimization-based motion planning under
  uncertainty.
\newblock \emph{IEEE Transactions on Automation Science and Engineering},
  13\penalty0 (2):\penalty0 437--447, 2016.

\bibitem[Sun et~al.(2018)Sun, Gordon, Boots, and Bagnell]{sun2018dual}
Sun, W., Gordon, G.~J., Boots, B., and Bagnell, J.
\newblock Dual policy iteration.
\newblock In \emph{Advances in Neural Information Processing Systems}, pp.\
  7059--7069, 2018.

\bibitem[Sun et~al.(2019)Sun, Jiang, Krishnamurthy, Agarwal, and
  Langford]{sun2018model}
Sun, W., Jiang, N., Krishnamurthy, A., Agarwal, A., and Langford, J.
\newblock Model-based {RL} in contextual decision processes: {PAC} bounds and
  exponential improvements over model-free approaches.
\newblock In \emph{Conference on Learning Theory}, pp.\  1--36, 2019.

\bibitem[Sutton(1996)]{sutton1996generalization}
Sutton, R.~S.
\newblock Generalization in reinforcement learning: Successful examples using
  sparse coarse coding.
\newblock \emph{Advances in neural information processing systems}, pp.\
  1038--1044, 1996.

\bibitem[Todorov \& Li(2005)Todorov and Li]{todorov2005generalized}
Todorov, E. and Li, W.
\newblock A generalized iterative lqg method for locally-optimal feedback
  control of constrained nonlinear stochastic systems.
\newblock In \emph{Proceedings of the 2005, American Control Conference,
  2005.}, pp.\  300--306. IEEE, 2005.

\bibitem[Todorov et~al.(2012)Todorov, Erez, and Tassa]{todorov2012mujoco}
Todorov, E., Erez, T., and Tassa, Y.
\newblock Mujoco: A physics engine for model-based control.
\newblock In \emph{2012 IEEE/RSJ International Conference on Intelligent Robots
  and Systems}, pp.\  5026--5033. IEEE, 2012.

\bibitem[Trott et~al.(2019)Trott, Zheng, Xiong, and Socher]{trott2019keeping}
Trott, A., Zheng, S., Xiong, C., and Socher, R.
\newblock Keeping your distance: Solving sparse reward tasks using
  self-balancing shaped rewards.
\newblock \emph{Advances in Neural Information Processing Systems},
  32:\penalty0 10376--10386, 2019.

\bibitem[Umlauft et~al.(2018)Umlauft, P{\"o}hler, and
  Hirche]{umlauft2018uncertainty}
Umlauft, J., P{\"o}hler, L., and Hirche, S.
\newblock An uncertainty-based control lyapunov approach for control-affine
  systems modeled by gaussian process.
\newblock \emph{IEEE Control Systems Letters}, 2\penalty0 (3):\penalty0
  483--488, 2018.

\bibitem[Wang et~al.(2019)Wang, Bao, Clavera, Hoang, Wen, Langlois, Zhang,
  Zhang, Abbeel, and Ba]{wang2019benchmarking}
Wang, T., Bao, X., Clavera, I., Hoang, J., Wen, Y., Langlois, E., Zhang, S.,
  Zhang, G., Abbeel, P., and Ba, J.
\newblock Benchmarking model-based reinforcement learning.
\newblock \emph{arXiv preprint arXiv:1907.02057}, 2019.

\bibitem[Williams et~al.(2017{\natexlab{a}})Williams, Aldrich, and
  Theodorou]{williams2017model}
Williams, G., Aldrich, A., and Theodorou, E.~A.
\newblock Model predictive path integral control: From theory to parallel
  computation.
\newblock \emph{Journal of Guidance, Control, and Dynamics}, 40\penalty0
  (2):\penalty0 344--357, 2017{\natexlab{a}}.

\bibitem[Williams et~al.(2017{\natexlab{b}})Williams, Wagener, Goldfain, Drews,
  Rehg, Boots, and Theodorou]{williams2017information}
Williams, G., Wagener, N., Goldfain, B., Drews, P., Rehg, J.~M., Boots, B., and
  Theodorou, E.~A.
\newblock Information theoretic mpc for model-based reinforcement learning.
\newblock In \emph{2017 IEEE International Conference on Robotics and
  Automation (ICRA)}, pp.\  1714--1721. IEEE, 2017{\natexlab{b}}.

\bibitem[Yang \& Wang(2019)Yang and Wang]{yang2019reinforcement}
Yang, L.~F. and Wang, M.
\newblock Reinforcement leaning in feature space: Matrix bandit, kernels, and
  regret bound.
\newblock \emph{arXiv preprint arXiv:1905.10389}, 2019.

\bibitem[Zhou et~al.(2020)Zhou, He, and Gu]{zhou2020provably}
Zhou, D., He, J., and Gu, Q.
\newblock Provably efficient reinforcement learning for discounted mdps with
  feature mapping.
\newblock \emph{arXiv preprint arXiv:2006.13165}, 2020.

\end{thebibliography}
\bibliographystyle{icml2021}

\newpage
\onecolumn
\appendix
% !TEX root =  main.tex 

\section{Sample Complexity Analysis of PC-MLP}
\label{app:proof}
%We can use the same algorithmic framework for System Identification by simply setting $r(s,a) = 0$ for all $s,a\in\Scal\times\Acal$. We require SysID to achieve the following two goals, (1) identify a reachable set $\Kcal\subset \Scal\times\Acal$, such that no policy in $\Pi$ can escape $\Kcal$ with a non-trivial probability; (2) learn a model $\widehat{P}$ such that given any new reward function $r$ with $r(s,a)\in [0,1]$ for all $(s,a)$, $\widehat{\pi} = \argmax_{\pi\in\Pi} J(\pi ; r, \widehat{P})$ is near optimal with respect to the best policy $\pi^\star = \argmax_{\pi\in\Pi} J(\pi ;r, P^\star )$

In this section, we provide a detailed analysis for PC-MLP on both KNRs and Linear MDPs. 

\subsection{Model Learning from MLE}
We consider a fixed episode $n$. The derived result here can be applied to all episodes via a union bound.  We derive model's generalization error under distribution $d^{\boldpi_n}$ in terms of total variation distance. For linear MDPs, we can directly call Theorem 21 from \cite{agarwal2020flambe}, which shows under realizability $P^\star\in \Pcal$, the empirical maximizer of MLE directly enjoys the following generalization error.
\begin{lemma}[MLE for linear MDPs (Theorem 21 from \cite{agarwal2020flambe})] Fix $\delta\in (0,1)$, and assume $|\Pcal| < \infty$ and $P^\star \in \Pcal$. Consider $M$ samples with $(s_i,a_i)\sim d^{\boldpi_n}$, and $s_i' \sim P^\star(\cdot | s,a)$. Denote the empirical risk minimizer as $\widehat{P}_n = \argmax_{P\in\Pcal} \sum_{i=1}^M \ln P(s_i' | s_i,a_i)$. We have that with probability at least $1-\delta$,
\begin{align*}
\mathbb{E}_{s,a\sim d^{\boldpi_n}} \left\| \widehat{P}_n(\cdot | s,a) - P^\star(\cdot | s,a)  \right\|^2_1 \leq {\frac{2\ln\left( \left\lvert  \Pcal \right\rvert /\delta \right)}{ M }}.
\end{align*}\label{lem:linear_mdp_model_error}
\end{lemma}

For KNRs, we do not need to rely on the exact empirical risk minimization. Instead, we can use approximation optimization approach SGD here. Note that due to the Gaussian noise in the KNR model, we have that for a model $P_W$ with $W$ as the parameterization: 
\begin{align*}
\ln P_W(s' | s,a) =  - \| W\phi(s,a) - s'  \|_2^2 / \sigma^2 + \ln (1/C),
\end{align*}  where $C$ is the normalization constant for Gaussian distribution with zero mean and covariance matrix $\sigma^2 I$.
Hence gradient of the log-likelihood is equivalent to the gradient of the squared $\ell_2$ loss. Specifically, for approximately optimizing the empirical log-likelihood, we start with ${W}^0 = 0$, and perform SGD with ${W}^{i+1} = \prod_{W:\|W\|_2\leq F} \left({W}^i - \eta \left( W^i \phi(s_i,a_i) - s'_i \right) \phi(s_i,a_i)^{\top}\right)$ and set $\widehat{W}_n = \frac{1}{M} \sum_{i=1}^M W^i$.

To use SGD's result, we first need to bound all states $s'$. The following lemma indicates that with high probability, the states have bounded norms. 

\begin{lemma}\label{lem:state_bound}
 In each episode, we generate M data points $\{s_i,a_i,s_i'\}_{i=1}^M$ with $s_i,a_i\sim d^{\boldpi_n}$ and $s_i'\sim \mathcal{N}\left(  W^\star\phi(s_i,a_i)  ,\sigma^2 I\right)$ with $\| W^\star \|_2 \leq F$. With probability at least $1-\delta$, we have:
\begin{align*}
\left\| s_i' \right\|_2 \leq F + \sigma\sqrt{ d_s \ln\left( d_s M /\delta \right)}, \forall i \in [1,\dots, M].
\end{align*}% for all $i \in [1,\dots, M]$
\end{lemma}
\begin{proof}Denote $s_i' = \epsilon_i + W^\star\phi(s_i,a_i)$ where $\epsilon_i\sim \Ncal(0,\sigma^2 I)$. 
We must have that for a fixed $i$ and dimension $j \in [d_s]$:
\begin{align*}
\mathbb{P}( |\epsilon_i [j] | \geq  t) \leq \exp( -t^2 / \sigma^2 ).
\end{align*}Take a union bound over all $i, j$, we have:
\begin{align*}
\mathbb{P}( \exists i, j, \text{ s.t. } |\epsilon_i [j] | \geq  t) \leq  d_s M  \exp( -t^2 / \sigma^2 ).
\end{align*}Set $d_s M \exp( -t^2 / \sigma^2 ) = \delta$, we get:
\begin{align*}
t = \sigma \sqrt{  \ln \left(\frac{ d_s M }{ \delta }  \right) }.
\end{align*}  Hence with probability $1-\delta$, for $i , j $, we have:
\begin{align*}
\left\lvert \epsilon_i [j] \right\rvert \leq  \sigma \sqrt{  \ln \frac{ d_s M  }{ \delta }   }.
\end{align*}%This concludes the proof. 
Hence, $\| s_i' \|_2 \leq \| W^\star\phi(s_i,a_i)\|_2 + \| \epsilon_i\|_2 \leq F + \sigma \sqrt{d_s \ln\left(\frac{ d_s M }{\delta}\right)}$. 
\end{proof}
Throughout this section, we assume the above event in Lemma~\ref{lem:state_bound} holds and we denote $F + \sigma\sqrt{ d_s \ln\left( d_s M /\delta \right)} = B$ for notation simplicity .

Now we can call Lemma~\ref{lem:sgd_dim_free} (dimension-free SGD result) to conclude the following lemma. 
\begin{lemma}[MLE on KNRs] With probability at least $1-\delta$, we have that:
\begin{align*}
\mathbb{E}_{s,a\sim d^{\boldpi_n}} \left[ \left\| \widehat{W}_n \phi(s,a)  - W^\star\phi(s,a) \right\|_2^2 \right] \leq  \frac{3 (F^2 + FB) \ln(1/\delta)}{\sqrt{M}},
%\mathbb{E}_{s,a\sim d^{\boldpi_n}} \left\| \widehat{P}_n(\cdot | s,a) - P^\star(\cdot | s,a) \right\|_1 \leq \frac{4}{\sigma} \frac{ \sqrt{(F^2 + FB)\ln(2/\delta)}}{ M^{1/4} }
\end{align*} where $B = F + \sigma \sqrt{d_s \ln\left( \frac{2d_s M}{\delta} \right)}$.
\label{lem:model_learning_knr}
\end{lemma}
\begin{proof}
Lemma~\ref{lem:state_bound} states that with probability $1-\delta/2$, we have that $\|s_i'\|_2 \leq B$ for all $i \in [1,\dots, M]$. Condition on this event, we call Lemma~\ref{lem:sgd_dim_free} and using the fact that $\|W^\star\|_2\leq \|W^\star\|_F \leq F$, we get that with probability at least $1-\delta/2$, 
\begin{align*}
\mathbb{E}_{s,a\sim d^{\boldpi_n}} \left[ \left\| \widehat{W}_n \phi(s,a)  - W^\star\phi(s,a) \right\|_2^2 \right]  \leq \frac{ 3(F^2 + FB) \ln(2/\delta) }{\sqrt{M}}
\end{align*} 
The total probability of failure is $\delta$. 

%For total variation distance between two Gaussians with identical covariance $\sigma^2 I$, we have:
%\begin{align*}
%\left\| \widehat{P}_n(\cdot | s,a) - P^\star(\cdot | s,a)  \right\|_{tv} \leq \frac{1}{\sigma} \| \widehat{W}_n\phi(s,a) - W^\star\phi(s,a)  \|_2.
%\end{align*}
%Via Jensen's inequality and the concavity of the square root function, we have:
%\begin{align*}
%\mathbb{E}_{s,a\sim d^{\boldpi_n}} \left[ \left\| \widehat{W}_n \phi(s,a)  - W^\star\phi(s,a) \right\|_2 \right] \leq \sqrt{ \mathbb{E}_{s,a\sim d^{\boldpi_n}} \left[ \left\| \widehat{W}_n \phi(s,a)  - W^\star\phi(s,a) \right\|_2^2 \right]   }
%\end{align*}
%Thus, we have:
%\begin{align*}
%\mathbb{E}_{s,a\sim d^{\boldpi_n}}\left\| \widehat{P}_n(\cdot | s,a) - P^\star(\cdot | s,a)  \right\|_{1}  & \leq \frac{2}{\sigma} \sqrt{ \mathbb{E}_{s,a\sim d^{\boldpi_n}} \left[ \left\| \widehat{W}_n \phi(s,a)  - W^\star\phi(s,a) \right\|_2^2 \right]   }\\
%& \leq \frac{4}{\sigma}  \sqrt{ 3(F^2 + FB) \ln(2/\delta) } /  M^{1/4}.
%\end{align*}
This concludes the proof. 
\end{proof}

\subsection{Optimism at the Starting State}

We denote $b_n(s,a)$ using $\Sigma_n $ as follows:
\begin{align*}
b_n(s,a) = \min\left\{ c \sqrt{\phi(s,a)^{\top} \Sigma_n^{-1} \phi(s,.a)} , \; H  \right\}.
\end{align*}
Recall that the reward bonus in the algorithm is defined with respect to $\widehat\Sigma_n$ in Eq.~\ref{eq:reward_bonus}. We will link $b_n$ and $\widehat{b}_n$ later in the analysis.

The bonus is related to the uncertainty in the model. Consider any policy $\pi$, reward bonus $b_n$ in the form of Eq.~\ref{eq:reward_bonus}, and any model $\widehat{P} \in \Pcal$. Denote $\widehat{V}^{\pi}_{r+b_n;h}$ as the value function at time step $h$ under policy $\pi$ in model $\widehat{P}$ and reward $r+b_n$. Note that $r(s,a) + b_n(s,a) \in [0, H+1]$, we have $\|\widehat{V}^{\pi}_{r+b_n,h}\|_{\infty} \leq H^2$ for any $h$. 

\begin{lemma}[Optimism in Linear MDPs] 
Assume the following condition hold for all $n \in [1,\dots, N]$: 
\begin{align*}
\EE_{s,a\sim d^{\boldpi_n}} \| \widehat{P}_n(\cdot | s,a) - P^\star(\cdot | s,a) \|^2_1 \leq \epsilon_{stat}\in\mathbb{R}^+, \forall n.
\end{align*}
Set $c = H \sqrt{\left(\lambda d + N \epsilon_{stat}\right)}$, and assume $b_n(s,a) \leq \widehat{b}_n(s,a) \leq 4 b_n(s,a)$ holds for all $n$.
Then, we have that for any $n$:% with probability at least $1-\delta$, 
\begin{align*}
\widehat{V}^{\pi_{n+1}}_{r+\widehat{b}_n;0}(s_0) \geq \max_{\pi\in \Pi} V^{\pi}_0(s_0).
\end{align*}\label{lem:optimism_linear_mdp}
\end{lemma}
\begin{proof}We denote $\pi^\star = \argmax_{\pi\in\Pi}V^{\pi}_0(s_0)$. We consider $\pi^\star$ specifically. Without loss of generality, we denote $\widehat{V}^{\pi}_{r+\hat{b}_n;H}(s) = V^{\pi}_H(s) = 0$ for all $s\in \Scal$. Thus, for $h = H$, we have $\widehat{V}^{\pi^\star}_{r+\hat{b}_n;H}(s) \geq V^{\pi^\star}_H(s), \forall s$. 

Assume that at time step $h+1$, we have $\widehat{V}^{\pi^\star}_{r+\hat{b}_n;h+1}(s) \geq V^{\pi^\star}_{h+1}(s), \forall s$. Now we move on to prove this also holds at time step $h$. Denote $a^\star = \pi^\star(s)$.  Also we define $\Sigma_n = \sum_{i=1}^{n} \Sigma_{\pi_n} + \lambda I$.
\begin{align*}
\widehat{V}^{\pi^\star}_{r+\hat{b}_n;h}(s) - {V}^{\pi^\star}_{h}(s) & = r(s,a^\star)+\widehat{b}_n(s,a^\star) + \mathbb{E}_{s'\sim \widehat{P}^n(\cdot | s,a^\star)} \widehat{V}^{\pi^\star}_{r+\hat{b}_n;h+1}(s') - \left( r(s,a^\star)+\EE_{s'\sim P^\star(\cdot|s,a^\star)} V^{\pi^\star}_{h+1}(s')  \right)\\
&  \geq  b_n(s,a^\star) + \left(\mathbb{E}_{s'\sim \widehat{P}^n(\cdot | s,a^\star)} {V}^{\pi^\star}_{h+1}(s') - \EE_{s'\sim P^\star(\cdot|s,a^\star)} V_{h+1}^{\pi^\star}(s')  \right) \\
& =  b_n(s,a^\star) + \left( (\widehat{\mu}_n  - \mu^\star) \phi(s,a^\star)   \right) \cdot V^{\pi^\star}_{h+1}
%& \geq b_n(s,a^\star) -  \|  \widehat{P}^n(\cdot|s,a^\star) - P^\star(\cdot | s,a^\star)  \|_{\Sigma_n} \| V^{\pi^\star}_{h+1} \|_{\Sigma_n^{-1}} \\
%& \geq b_n(s,a^\star) -  \|  \widehat{P}^n(\cdot|s,a^\star) - P^\star(\cdot | s,a^\star)  \|_{\Sigma_n} \| 
\end{align*}
We bound $\left( (\widehat{\mu}_n  - \mu^\star) \phi(s,a)   \right) \cdot V^{\pi^\star}_{h+1}$ below. 
\begin{align*}
&\left\lvert \left( (\widehat{\mu}_n  - \mu^\star) \phi(s,a)   \right) \cdot V^{\pi^\star}_{h+1} \right\rvert^2 \leq \|\phi(s,a)\|^2_{\Sigma^{-1}_n} \left\| \left(\widehat{\mu}_n - \mu^\star\right)^{\top} V^{\pi^\star}_{h+1} \right\|^2_{\Sigma_n} \\
&= \|\phi(s,a)\|^2_{\Sigma^{-1}_n} \left( \lambda \left\| \left(\widehat{\mu}_n - \mu^\star\right)^{\top} V^{\pi^\star}_{h+1}   \right\|_2^2 + n \mathbb{E}_{s,a\sim d^{\boldpi_n}} \left( \phi(s,a)^{\top} \left(\widehat{\mu}_n - \mu^\star\right)^{\top} V^{\pi^\star}_{h+1}   \right)^2   \right) \\
& \leq \|\phi(s,a)\|^2_{\Sigma^{-1}_n} \left(\lambda \left\| \left(\widehat{\mu}_n - \mu^\star\right)^{\top} V^{\pi^\star}_{h+1}   \right\|_2^2 +  n H^2 \mathbb{E}_{s,a\sim d^{\boldpi_n}} \| \widehat{P}_n(\cdot|s,a) - P^\star(\cdot|s,a) \|^2_1  \right) \\
& \leq \|\phi(s,a)\|^2_{\Sigma^{-1}_n}  \left( \lambda H^2 d + nH^2 \mathbb{E}_{s,a\sim d^{\boldpi_n}} \| \widehat{P}_n(\cdot|s,a) - P^\star(\cdot|s,a) \|^2_1   \right) \\
& \leq \left( \lambda H^2 d + nH^2 \epsilon_{stat} \right) \|\phi(s,a)\|^2_{\Sigma_n^{-1}} \leq c^2 \| \phi(s,a) \|^2_{\Sigma_n^{-1}}.
\end{align*}
%Note that the second term above is exactly the generalization error of model under the training distribution $d^{\boldpi_n}$.
Thus, we get that:
\begin{align*}
\widehat{V}^{\pi^\star}_{r+\hat{b}_n;h}(s) - {V}^{\pi^\star}_{h+1}(s) \geq b_n(s,a^\star) -  c\|\phi(s,a^\star)\|_{\Sigma_n^{-1}} = 0. 
\end{align*} Note that the above holds for any $s$. Thus via induction, we conclude that at $h=0$, we have $\widehat{V}^{\pi^\star}_{r+\hat{b}_n;0}(s_0) \geq V^{\pi^\star}_{0}(s_0)$.
Using the fact that $\pi_{n+1} = \argmax_{\pi\in\Pi} \widehat{V}^{\pi}_{r+\hat{b}_n;0}(s_0)$, we conclude the proof. 
\end{proof}

\begin{lemma}[Optimism in KNRs] 
Assume the following condition hold for all $n \in [1,\dots, N]$: 
\begin{align*}
%\EE_{s,a\sim d^{\boldpi_n}} \| \widehat{P}_n(\cdot | s,a) - P^\star(\cdot | s,a) \|^2_1 \leq \epsilon_{stat}\in\mathbb{R}^+, \forall n.
\EE_{s,a\sim d^{\boldpi_n}} \left\| \widehat{W}_n\phi(s,a) - W^\star\phi(s,a) \right\|^2_{2} \leq \epsilon_{stat}, \forall n, 
\end{align*}
Set $c = \frac{1}{\sigma}{H \sqrt{ \lambda 4 F^2  + N \epsilon_{stat}}}$, and assume that $b_n(s,a) \leq \widehat{b}_n(s,a) \leq 4 b_n(s,a)$ holds for all $n$.
we have that for any $n$:% with probability at least $1-\delta$, 
\begin{align*}
\widehat{V}^{\pi_{n+1}}_{r+\hat{b}_n;0}(s_0) \geq \max_{\pi\in \Pi} V^{\pi}_0(s_0).
\end{align*}\label{lem:optimism_knr}
\end{lemma}
\begin{proof}
For any $n$, the condition in the lemma implies that:
\begin{align*}
\sum_{i=1}^n \EE_{s,a\sim d^{\pi_i}} \left\| \widehat{W}_n\phi(s,a) - W^\star\phi(s,a) \right\|^2_{2}  = \tr\left( \left(\widehat{W}_n - W^\star\right)\sum_{i=1}^n \Sigma_{\pi_i} \left(\widehat{W}_n - W^\star\right)^{\top} \right)
\leq n \epsilon_{stat}
\end{align*} Note that $\Sigma_n = \sum_{i=1}^n \Sigma_{\pi_i} + \lambda I$, we have that:
\begin{align*}
\left\| \left(\widehat{W}_n - W^\star\right)\Sigma_n^{1/2} \right\|_2^2 \leq \tr\left(  \left(\widehat{W}_n - W^\star\right) \Sigma_n \left(\widehat{W}_n - W^\star\right)^{\top} \right) \leq n\epsilon_{stat} + \lambda 4 F^2,
\end{align*} where we use the norm bound that $\|\widehat{W}_n\|^2_F \leq F^2, \|W^\star\|_F^2 \leq F^2$.

Similarly, we can use induction to prove optimism. Assume $\widehat{V}^{\pi^\star}_{r+\hat{b}_n; h+1}(s) \geq V^{\pi^\star}_{h+1}(s)$ for all $s$. For any $s$, denote $a^\star = \pi^\star(s)$, we have:
\begin{align*}
\widehat{V}^{\pi^\star}_{r+\hat{b}_n; h}(s)  - V^{\pi^\star}_{h}(s)&  \geq   b_n(s,a^\star) + \left(\mathbb{E}_{s'\sim \widehat{P}^n(\cdot | s,a^\star)} {V}^{\pi^\star}_{h+1}(s') - \EE_{s'\sim P^\star(\cdot|s,a^\star)} V_{h+1}^{\pi^\star}(s')  \right) \\
& \geq b_n(s,a^\star) - \left\| \widehat{P}_n(\cdot | s,a^\star) - P^\star(\cdot | s,a^\star) \right\|_1 \|V^{\pi^\star}_{h+1}\|_{\infty} \\
& \geq b_n(s,a^\star) - \frac{H}{\sigma} \left\| \left(\widehat{W}_n - W^\star\right)\phi(s,a^\star) \right\|_2\\
& \geq b_n(s,a^\star) - \frac{H}{\sigma} \left\| \widehat{W}_n - W^\star \right\|_{\Sigma_n} \left\| \phi(s,a^\star) \right\|_{\Sigma_n^{-1}}\\
& \geq b_n(s,a^\star) - \frac{H \sqrt{n\epsilon_{stat} + \lambda 4 F^2 }}{\sigma} \|\phi(s,a^\star)\|_{\Sigma_n^{-1}} \geq 0,
\end{align*} due to the set up of $c$.  Similar via induction, this concludes the proof. 
\end{proof}

\subsection{Regret Upper Bound}

Below we consider bounding $\sum_{n=1}^N \left( J(\pi^\star; P^\star) - J(\pi_n; P^\star) \right)$ using optimism we proved in the section above. 
\begin{lemma}[Regret bound in linear MDPs]Assuming all conditions in Lemma~\ref{lem:optimism_linear_mdp} holds.  We have:
\begin{align*}
\sum_{n=1}^N \left(J(\pi^\star; r, P^\star) - J(\pi_n; r, P^\star)\right) \leq  6 H^2  \sum_{n=1}^{N-1} \EE_{s,a\sim d^{\pi_{n+1}}} \left[ b_n(s,a)\right] +  H.
\end{align*}\label{lem:regret_linear_mdp}
\end{lemma}
\begin{proof}
Since the condition in Lemma~\ref{lem:optimism_linear_mdp} holds, we have that for all $n$, optimism holds, i.e., $J(\pi_{n+1}; r+\widehat{b}_n, \widehat{P}_n) \geq J(\pi^\star; r, P^\star)$.  Hence, together with the simulation lemma (Lemma~\ref{lem:simulation}) we have:
\begin{align*}
 J(\pi^\star;r, P^\star) - J(\pi_{n+1}; r, P^\star) & \leq J(\pi_{n+1}; r+\widehat{b}_n, \widehat{P}_n) - J(\pi_{n+1}; r, P^\star) \\
 &= \sum_{h=0}^{H-1} \EE_{s,a\sim d^{\pi_{n+1}}_h }\left[ \widehat{b}_n(s,a) + \left(\widehat{P}_n(\cdot | s,a) - P^\star(\cdot|s,a)\right)\cdot \widehat{V}^{\pi_{n+1}}_{r+\hat{b}_n;h+1}\right]\\
 & \leq \sum_{h=0}^{H-1} \EE_{s,a\sim d^{\pi_{n+1}}_h }\left[ 4{b}_n(s,a) + \left(\widehat{P}_n(\cdot | s,a) - P^\star(\cdot|s,a)\right)\cdot \widehat{V}^{\pi_{n+1}}_{r+\hat{b}_n;h+1}\right].
\end{align*} Note that $\| \widehat{V}^{\pi}_{r+\hat{b};h} \|_{\infty} \leq H^2$ for any $\pi,h, n$. Following similar derivation in the proof of Lemma~\ref{lem:optimism_linear_mdp},  we have:
\begin{align*} 
\left\lvert \left( (\widehat{\mu}_n  - \mu^\star) \phi(s,a)   \right) \cdot \widehat{V}^{\pi_{n+1}}_{r+b_n;h+1} \right\rvert & \leq \min\left\{ \|\phi(s,a)\|_{\Sigma_n^{-1}} \sqrt{\left( \lambda H^4 d + nH^4 \epsilon_{stat} \right) }, 2 H^2 \right\}\\
&\leq 2 H \min\left\{ H \sqrt{\lambda d + N \epsilon_{stat} } \cdot \left\|\phi(s,a)\right\|_{\Sigma_n^{-1}},\; H\right\} = 2 H b_n(s,a).
\end{align*} Note that the regret at the first policy $\pi_1$ is at most $H$. This concludes the proof. 
\end{proof}

\begin{lemma}[Regret bound in KNRs]Assuming all conditions in Lemma~\ref{lem:optimism_knr} holds.  We have:
\begin{align*}
\sum_{n=1}^N \left(J(\pi^\star; r, P^\star) - J(\pi_n; r, P^\star)\right) \leq  5 H^2  \sum_{n=1}^{N-1} \EE_{s,a\sim d^{\pi_{n+1}}} \left[ b_n(s,a)\right] + H.
\end{align*}\label{lem:regret_knr}
\end{lemma}
\begin{proof}
Again, via simulation lemma and optimism, we have:
\begin{align*}
J(\pi^\star;r, P^\star) - J(\pi_{n+1}; r, P^\star)& \leq \sum_{h=0}^{H-1} \EE_{s,a\sim d^{\pi_{n+1}}_h }\left[ \widehat{b}_n(s,a) + \left(\widehat{P}_n(\cdot | s,a) - P^\star(\cdot|s,a)\right)\cdot \widehat{V}^{\pi_{n+1}}_{r+b_n;h+1}\right]\\
& \leq \sum_{h=0}^{H-1} \EE_{s,a\sim d^{\pi_{n+1}}_h }\left[ 4 {b}_n(s,a) + \left(\widehat{P}_n(\cdot | s,a) - P^\star(\cdot|s,a)\right)\cdot \widehat{V}^{\pi_{n+1}}_{r+b_n;h+1}\right]
\end{align*}
Following the derivation in the proof of Lemma~\ref{lem:optimism_knr}, we have:
\begin{align*}
\left\lvert \left(\widehat{P}_n(\cdot | s,a) - P^\star(\cdot|s,a)\right)\cdot \widehat{V}^{\pi_{n+1}}_{r+b_n;h+1} \right\rvert \leq \frac{H^2 \sqrt{n\epsilon_{stat} + \lambda 4 F^2 d_s }}{\sigma} \|\phi(s,a)\|_{\Sigma_n^{-1}} = H b_n(s,a).
\end{align*}Combine the above two inequalities, we conclude the proof. 
\end{proof}

Recall the definition of information gain, $\mathcal{I}_N(\lambda) = \max_{\pi_1,\dots,\pi_N} \ln \det \left(  I + \frac{1}{\lambda} \sum_{n=1}^N \Sigma_{\pi_n}  \right) $.

\begin{lemma}
For any sequence of policies $\pi_1,\dots, \pi_n$, with $\Sigma_n = \sum_{i=1}^n \Sigma_{\pi_n}$ and $b_n(s,a) \leq c \sqrt{\phi(s,a)^{\top} \Sigma_n^{-1} \phi(s,a)}$, we have that:
\begin{align*}
\sum_{n=1}^{N-1} \EE_{s,a\sim d^{\pi_{n+1}}} b_n(s,a) \leq c \sqrt{2 N \mathcal{I}_N(\lambda)}.
\end{align*} When $\phi\in\mathbb{R}^d$, we have:
\begin{align*}
\sum_{n=1}^{N-1} \EE_{s,a\sim d^{\pi_{n+1}}} b_n(s,a) \leq c \sqrt{2N d\ln\left( 1 + N /\lambda \right) }.
\end{align*}\label{lem:bonus_sum}
\end{lemma}
\begin{proof}
Starting from the definition of $b_n$, we have:
\begin{align*}
\sum_{n=1}^{N-1} \EE_{s,a\sim d^{\pi_{n+1}}} b_n(s,a) &  \leq c  \sum_{n=1}^{N-1} \EE_{s,a\sim d^{\pi_{n+1}}} \sqrt{ \phi(s,a)^{\top} \Sigma_n^{-1} \phi(s,a)} \leq c \sqrt{N} \sqrt{ \sum_{n=1}^{N-1} \EE_{s,a\sim d^{\pi_{n+1}}} { \phi(s,a)^{\top} \Sigma_n^{-1} \phi(s,a)} } \\
& = c \sqrt{N} \sqrt{\sum_{n=1}^{N-1} \tr\left( \Sigma_{\pi_{n+1}} \Sigma_n^{-1} \right)} \leq c \sqrt{2 N \ln\left( \det( \Sigma_N) / \det(\lambda I ) \right) } \\
%& \leq \sqrt{2N d\ln\left( 1 + N /\lambda \right) }.
\end{align*} Here in the second inequality, we use Cauchy-Schwartz inequality, in the third inequality, we use Lemma~\ref{lem:trace_tele}.

If $\phi\in\mathbb{R}^d$, we have that:
\begin{align*}
\sqrt{N \ln(\det( \Sigma_N ) / \det(\lambda I))} \leq \sqrt{N d\ln(1 + N / \lambda)},
\end{align*} where we use $\|\phi(s,a)\|_2 \leq 1$. 

This concludes the proof. 
\end{proof}

\subsection{Concluding the Sample Complexity Calculation}

Before concluding the final sample complexity, we need to link $\widehat{\Sigma}_n$ to $\Sigma_n$, as our reward bonus in the algorithm is defined in terms of the empirical estimate $\widehat\Sigma_n$.

\begin{lemma}[Relating $\widehat{b}_n$ and $b_n$] With probability at least $1-\delta$, for all $n\in [1,\dots, N]$, we have:
\begin{align*}
b_n(s,a) \leq \widehat{b}_n(s,a) \leq 4 b_n(s,a), \forall s,a.
\end{align*}\label{lem:empirical_bonus}
\end{lemma}
\begin{proof}
The proof uses Lemma~\ref{lem:inverse_covariance}. Under Lemma~\ref{lem:inverse_covariance}, we have:
\begin{align*}
& \min\left\{ c \sqrt{\phi(s,a)^{\top}\widehat{\Sigma}^{-1}_n \phi(s,a)}, H  \right\} \leq \min\left\{ c\sqrt2 \sqrt{\phi(s,a)^{\top}{\Sigma}^{-1}_n \phi(s,a)}, H  \right\} \\
 &\leq \sqrt{2}  \min\left\{ c\sqrt{\phi(s,a)^{\top}{\Sigma}^{-1}_n \phi(s,a)}, H  \right\} = \sqrt{2} b_n(s,a).
\end{align*} and
\begin{align*}
&b_n(s,a) / \sqrt{2}= \min\left\{ c\sqrt{ \phi(s,a)^{\top}\Sigma_n^{-1} \phi(s,a)},H \right\} / \sqrt{2}\\
& \leq \min\left\{ c\sqrt{(1/2) \phi(s,a)^{\top}\Sigma_n^{-1} \phi(s,a)},H \right\} \leq \min\left\{ c\sqrt{ \phi(s,a)^{\top}\widehat\Sigma_n^{-1} \phi(s,a)},H \right\},
\end{align*} 
Note that $\widehat{b}_n(s,a) = 2 \min\{c\sqrt{(1/2) \phi(s,a)^{\top}\widehat{\Sigma}_n^{-1} \phi(s,a)},H   \}$,
we  have that:
\begin{align*}
b_n(s,a) / \sqrt{2 }\leq \widehat{b}_n(s,a)/2 \leq \sqrt{2} b_n(s,a),
\end{align*} which concludes the proof.
%where again $b_n(s,a) = \min\{ c\sqrt{\phi(s,a)^{\top}\Sigma_n\phi(s,a)}, H \}$.
\end{proof}

In high level, from Lemma~\ref{lem:regret_linear_mdp}, Lemma~\ref{lem:regret_knr}, and Lemma~\ref{lem:bonus_sum} we know that after $N$ iterations, we have:
\begin{align*}
\max_{i\in [1,\dots, N]} J(\pi_i; r, P^\star) \geq J(\pi^\star; r, P^\star)  - \frac{10 H^2 c}{\sqrt{N}} \sqrt{ d \ln(1+N / \lambda)  }.
\end{align*}
Hence, to ensure we get an $\epsilon$ near optimal policy, we just need to set $N$ large enough such that $\frac{10 H^2 c}{\sqrt{N}} \sqrt{ d \ln(1+N / \lambda) \ln } \approx \epsilon$ and to do so, we need to control $M$ in order to make $c$ scale as a constant.

\subsubsection{Concluding for Linear MDPs}

\begin{theorem}[Sample Complexity for Linear MDPs] Set $\delta \in (0, 0.5)$ and $\epsilon \in (0,1)$. There exists a set of hyper-parameters, 
\begin{align*}
N = \frac{80 H^6 d^2}{\epsilon^2} \ln\left( \frac{40H^6 d^2}{\epsilon^2} \right), \quad  M = 2N\ln\left( |\Pcal| N / \delta \right), \quad c = H \sqrt{d + 1}, \quad K = 32 N^2 \ln\left( 8 N d / \delta \right), 
\end{align*} such that with probability at least $1-2\delta$, PC-MLP returns a policy $\widehat{\pi}$ such that:
\begin{align*}
J(\hat{\pi} ; r, P^\star) \geq \max_{\pi\in\Pi}J(\pi; r, P^\star) - \epsilon,
\end{align*} with number of samples
\begin{align*}
O\left( \frac{ H^{18} d^6 }{\epsilon^6} \cdot  \ln\left( \frac{  |\Pcal|  H^6 d^2}{ \epsilon^2 \delta}  \ln\left( \frac{H^6 d^2}{\epsilon^2} \right) \right)  \ln^3\left( \frac{H^6 d^2}{\epsilon^2} \right)    \right).
\end{align*}Ignoring log terms, we get the sample complexity scales in the order of $\widetilde{O}\left( \frac{H^{18} d^6}{ \epsilon^6} \right)$.
\label{thm:detailed_linear_mdp}
\end{theorem}
The above theorem verifies Theorem~\ref{thm:main_them_linear_mdp}

\begin{proof}
From Lemma~\ref{lem:optimism_linear_mdp}, we know that:
\begin{align*}
c = H \sqrt{d + N \epsilon_{stat}},
\end{align*} where we have set $\lambda = 1$ explicitly. Also from Lemma~\ref{lem:linear_mdp_model_error}, we know that with probability at least $1-\delta$, 
\begin{align*}
\epsilon_{stat} =  \frac{2 \ln( |\Pcal| N / \delta)}{M}.
\end{align*}
We set $M$ large enough such that $N \epsilon_{stat} = 1$. To achieve this, it is easy to verify that it is enough to set $M = 2 N\ln(|\Pcal| N / \delta)$. As $d \geq 1$, we immediately have that $c \leq 2H\sqrt{d}$ in this case.

To achieve $\epsilon$ approximation error, we set $N$ big enough such that:
\begin{align*}
\frac{10 H^2 c}{\sqrt{N}} \sqrt{ d \ln(1+N / \lambda) \ln } \leq \epsilon.
\end{align*} Using $c = 2H\sqrt{d}$, we get:
\begin{align*}
\frac{ 20 H^3 {d} }{\sqrt{N}} \sqrt{\ln(1+N)} \leq \epsilon.
\end{align*} 
We can verify that the above inequality holds when we set:
\begin{align*}
N  = \frac{80 H^6 d^2}{\epsilon^2} \ln\left( \frac{40H^6 d^2}{\epsilon^2} \right).
\end{align*}
Hence, the total number of samples we use for estimating models during $N$ epsilons is bounded as:
\begin{align*}
N \times M = 2N^2 \ln(|\Pcal| N / \delta) \leq   \frac{ H^{12} d^4 }{\epsilon^4}  \cdot 12800\ln^2\left( \frac{40H^6 d^2}{\epsilon^2} \right) \ln\left( \frac{ 80 H^6 d^2 |\Pcal| }{ \epsilon^2 \delta}  \ln\left( \frac{40H^6 d^2}{\epsilon^2} \right) \right).
\end{align*}
We also need to count the total number of samples used to estimate the covariance matrix $\widehat{\Sigma}_n$ for all $n$. From Lemma~\ref{lem:inverse_covariance}. The number is bounded as:
\begin{align*}
K \cdot N = 32 N^3 \ln\left( 8 N d / \delta \right) = \frac{ H^{18} d^{6}  }{\epsilon^6} \cdot (32\times 80^3) \ln^3\left( \frac{40 H^6 d^2}{\epsilon^2} \right) \ln\left( \frac{ 640 H^6 d^3 }{\epsilon^2 \delta} \ln\left( \frac{40H^6 d^2}{\epsilon^2} \right)  \right).
\end{align*}
The total number of samples is $N\times M + N \times K$, which after rearranging terms, we get:
\begin{align*}
NM + NK \leq \frac{ H^{18} d^6 }{\epsilon^6} \cdot  \ln\left( \frac{ 640 H^6 d^2 |\Pcal| }{ \epsilon^2 \delta}  \ln\left( \frac{40H^6 d^2}{\epsilon^2} \right) \right)  \ln^3\left( \frac{ 40H^6 d^2}{\epsilon^2} \right) \cdot (60\times80^3).
\end{align*}
We conclude here by noting that the total failure probability is at most $2\delta$.
\end{proof}

\subsubsection{Concluding for KNRs}
\begin{theorem}[Sample Complexity for KNRs] Set $\delta \in (0, 0.5)$ and $\epsilon \in (0,1)$. There exists a set of hyper-parameters, 
\begin{align*}
&N = \Theta \left( \frac{ H^6 F^2 d_s d }{ \sigma^2 \epsilon^2 } \ln\left( \frac{  H^6 F^2 d_s d }{ \sigma^2 \epsilon^2 } \right)\right), \quad  M  =  \Theta\left( N^2 (F^2 + FB)^2 \ln^2\left( \frac{N}{\delta} \right)\right),\\
&  c = \Theta\left( \frac{H}{\sigma} \sqrt{d + 1}\right), \quad K = 32 N^2 \ln\left( 8 N d / \delta \right), 
\end{align*} such that with probability at least $1-2\delta$, PC-MLP returns a policy $\widehat{\pi}$ such that:
\begin{align*}
J(\hat{\pi} ; r, P^\star) \geq \max_{\pi\in\Pi}J(\pi; r, P^\star) - \epsilon,
\end{align*} with number of samples
\begin{align*}
O\left( \frac{H^{18} d^3 d_s^3    }{\sigma^6\epsilon^6}  \cdot \left(  F^{10} + \sigma^2 F^8 d_s \right)  \nu    \right),
%O\left( \frac{ H^{18} d^6 }{\epsilon^6} \cdot  \ln\left( \frac{  |\Pcal|  H^6 d^2}{ \epsilon^2 \delta}  \ln\left( \frac{H^6 d^2}{\epsilon^2} \right) \right)  \ln^3\left( \frac{H^6 d^2}{\epsilon^2} \right)    \right).
\end{align*} where $\nu$ only contains log terms, 
\begin{align*}
\nu &= \ln^2\left( \frac{N}{\delta} \right) \ln\left( \frac{2d_s N}{\delta} \right) + \ln\left( \left( 8F^4 + 9 \sigma^2 F^2 d_s \ln\left(\frac{2d_s N}{\delta}\right) \right) \ln^2\left(\frac{N}{\delta}\right) N^2 +  18 \sigma^2 F^2 d_s  \right) \\
& \qquad + \ln^3 \left( \frac{ 6400 H^6 F^2 d_s d }{ \sigma^2 \epsilon^2 } \right) \ln\left( Nd / \delta \right).
\end{align*}Ignoring log terms, we get sample complexity scales in the order $\widetilde{O}\left( \frac{H^{18} d^3 d_s^3}{\sigma^6 \epsilon^6} \cdot\left( F^{10} + \sigma^2 F^{8} d_s \right) \right)$.
\label{thm:detailed_knr}
\end{theorem}
The above theorem verifies Theorem~\ref{thm:main_knr}

\begin{proof}
The proof is similar to the proof of Theorem~\ref{thm:detailed_linear_mdp}. From Lemma~\ref{lem:optimism_knr}, we know that $c \leq \frac{4 H}{\sigma} \sqrt{ \lambda F^2 d_{s}/\sigma^2 + N \epsilon_{stat}/\sigma^2}$. Also from Lemma~\ref{lem:model_learning_knr}, we know that 
\begin{align*}
\epsilon_{stat} = \frac{ 3(F^2 + F B)\ln(N / \delta) }{ \sqrt{M} },
\end{align*} with $B = F + \sigma \sqrt{ d_s \ln(2 d_s N M / \delta)}$, for all $n$. We set $M$ such that $N \epsilon_{stat} = 1$. This gives us that:
\begin{align*}
M \geq  9 N^2 (F^2 + FB)^2 \ln^2\left( \frac{N}{\delta} \right).
\end{align*} Solve for $M$, we can verify that it suffices to set $M$ as:
\begin{align*}
M & = \left( 8F^4 + 9 \sigma^2 F^2 d_s \ln\left(\frac{2d_s N}{\delta}\right) \right) \ln^2\left(\frac{N}{\delta} \right) N^2 \\
& \quad + 18 \sigma^2 F^2 d_s \ln\left( \left( 8F^4 + 9 \sigma^2 F^2 d_s \ln\left(\frac{2d_s N}{\delta}\right) \right) \ln^2\left(\frac{N}{\delta}\right) N^2 +  18 \sigma^2 F^2 d_s  \right)
\end{align*}This gives us $c = \frac{8H}{\sigma} \sqrt{ F^2 d_s} $.

Similarly, we will set $N$ such that $\frac{10 H^2 c}{\sqrt{N}} \sqrt{ d \ln(1+N / \lambda) } \leq \epsilon$. With $c = \frac{8H}{\sigma} \sqrt{ F^2 d_s}$, we get:
\begin{align*}
\frac{80 H^3 \sqrt{F^2 d_s} }{\sigma\sqrt{N}} \sqrt{ d \ln(1+N / \lambda) } \leq \epsilon
\end{align*}
We can verify that it suffices to set $N$ as:
\begin{align*}
N = \frac{ 12800 H^6 F^2 d_s d }{ \sigma^2 \epsilon^2 } \ln\left( \frac{ 6400 H^6 F^2 d_s d }{ \sigma^2 \epsilon^2 } \right)
\end{align*}
Thus the total number of samples used for model learning is upper bounded as:
\begin{align*}
N M  & = O\left(N^3 \left(  F^4 +  \sigma^2 F^2 d_s  \right) \ln^2\left( \frac{N}{\delta} \right) \ln\left( \frac{2d_s N}{\delta} \right)\right) \\
& \qquad + O\left( N \sigma^2 F^2 d_s \cdot \ln\left( \left( 8F^4 + 9 \sigma^2 F^2 d_s \ln\left(\frac{2d_s N}{\delta}\right) \right) \ln^2\left(\frac{N}{\delta}\right) N^2 +  18 \sigma^2 F^2 d_s  \right)  \right) \\
& = O\left( \nu \frac{ H^{18} d^3 d_s^3  F^6 (F^4 + \sigma^2 F^2 d_s)  }{\sigma^6 \epsilon^6}    \right), 
\end{align*} where $v$ only contains log terms, i.e., 
\begin{align*}
\nu = \ln^2\left( \frac{N}{\delta} \right) \ln\left( \frac{2d_s N}{\delta} \right) + \ln\left( \left( 8F^4 + 9 \sigma^2 F^2 d_s \ln\left(\frac{2d_s N}{\delta}\right) \right) \ln^2\left(\frac{N}{\delta}\right) N^2 +  18 \sigma^2 F^2 d_s  \right).
\end{align*}
We also need to count the total number of samples used to estimate the covariance matrix $\widehat{\Sigma}_n$ for all $n$. From Lemma~\ref{lem:inverse_covariance}. The number is bounded as:
\begin{align*}
K \cdot N = O\left( N^3 \ln\left( N d / \delta \right)\right) = O\left( \nu_1\cdot \frac{ H^{18} F^6 d_s^3 d^3 }{\sigma^6 \epsilon^6}  \right)
%= \frac{ H^{18} d^{6}  }{\epsilon^6} \cdot (32\times 80^3) \ln^3\left( \frac{40 H^6 d^2}{\epsilon^2} \right) \ln\left( \frac{ 640 H^6 d^3 }{\epsilon^2 \delta} \ln\left( \frac{40H^6 d^2}{\epsilon^2} \right)  \right).
\end{align*} where $\nu_1$ only contains log terms, i.e.,
\begin{align*}
\nu_1 = \ln^3 \left( \frac{ 6400 H^6 F^2 d_s d }{ \sigma^2 \epsilon^2 } \right) \ln\left( Nd / \delta \right).
\end{align*}
Combine the two terms, we can conclude that:
\begin{align*}
KN + KM = O\left( \frac{H^{18} d^3 d_s^3    }{\sigma^6\epsilon^6}  \cdot \left(  F^{10} + \sigma^2 F^8 d_s \right)  \left(\nu + \nu_1 \right)   \right).
\end{align*}
\end{proof}

\section{Auxiliary Lemmas}

\begin{lemma}[Dimension-free SGD (Lemma G.1 from \cite{agarwal2020pc})]
%\label{lemma:least_square_dim_free} 
Consider the following learning process. Initialize $W_1 =  \mathbf{0}$.  For $i = 1,\dots, N$, draw $x_i, y_i \sim \nu$, $\|y_i\|_2 \leq B$, $\|x_i\| \leq 1$;
Set $W_{i+1} =\prod_{\Wcal:=\{W:\|W\|_2\leq F \}} \left(W_{i} - \eta_i \left(W_i x_i  - y_i\right) x_i^{\top}\right)$ with $\eta_i = (F^2)/((F+B)\sqrt{N})$. Set $\widehat{W} = \frac{1}{N} \sum_{i=1}^{N} W_i$, we have that with probability at least $1-\delta$:
\begin{align*}
\EE_{x\sim \nu}\left[\left\|\widehat{W}\cdot x - \EE\left[y|x\right]\right\|_2^2\right] \leq  \EE_{x\sim \nu}\left[\left\| W^\star\cdot x - \EE\left[y|x\right] \right\|_2^2\right] +\frac{R\sqrt{\ln(1/\delta)}}{\sqrt{N}},
\end{align*} with any $W^\star$ such that $\|W^\star\|_2\leq F$ and $R = 3(F^2 + FB)$. %which is dimension free and only depends on the norms of the feature and $\theta^\star$ and the bound on $y$.
\label{lem:sgd_dim_free}
\end{lemma}

\begin{lemma}[Total Variation Distance between Two Gaussians]
Given two Gaussian distributions $P_1 = \Ncal(\mu_1, \sigma^2 I)$ and $P_2 = \Ncal(\mu_2, \sigma^2 I)$, we have $\| P_1 - P_2  \|_{tv} \leq \min\{ \frac{1}{\sigma} \| \mu_1 - \mu_2 \|_2, 1 \}$.
\end{lemma}
The above lemma can be verified using the KL divergence between two Gaussians and the application of Pinsker's inequality \citep{devroye2018total}.

\begin{lemma}
\label{lem:trace_tele}
Consider the following process.  For $n = 1, \dots, N$, $M_n = M_{n-1} + \Sigma_{n}$ with $M_{0} = \lambda \mathbf{I}$ and $\Sigma_n$ being PSD matrix with eigenvalues upper bounded by $1$. We have that:
\begin{align*}
2 \log\det( M_N) - 2 \log\det( \lambda\mathbf{I})  \geq  \sum_{n=1}^N \tr\left( \Sigma_{i} M_{i-1}^{-1} \right).
\end{align*}
\end{lemma}
The proof of the above lemma is standard and can be found in Lemma G.2 from \cite{agarwal2020pc} for instance. 

\begin{lemma}[Simulation Lemma] Consider a MDPs $\Mcal_1 = \{\hat{r}, \widehat{P}\}$ where $\hat{r}$ and $\widehat{P}$ represent reward and transition. For any policy $\pi:\Scal\times\Acal\mapsto \Delta(\Acal)$, we have:
\begin{align*}
J(\pi; \hat{r}, \widehat{P}) - J(\pi; r, P^\star) = \sum_{h=0}^{H-1} \mathbb{E}_{s,a\sim d_h^{\pi}} \left[ r'(s,a) - r(s,a) +  \mathbb{E}_{s'\sim \widehat{P}(\cdot | s,a)}\widehat{V}^{\pi}_h(s') - \EE_{s'\sim P^\star(\cdot|s,a)} \widehat{V}^{\pi}_h(s') \right].
\end{align*}\label{lem:simulation}
\end{lemma}
Simulation lemma is widely used in proving sample complexity for RL algorithms. For proof, see Lemma 10 in \cite{sun2018model} for instance. 

The following lemma studies the concentration related to the empirical covariance matrix $\widehat{\Sigma}_n$ and $\Sigma_n$.
\begin{lemma}[ Concentration with the Inverse of Covariance Matrix (Lemma G.4 from \cite{agarwal2020pc})]
\label{lem:inverse_covariance}
Consider a fixed $N$. Assume $\phi\in\mathbb{R}^d$. Given $N$ distributions $\nu_1,\dots, \nu_N$ with $\nu_i\in\Delta(\Scal\times\Acal)$, assume we draw $K$ i.i.d samples from $\nu_i$ and form $\widehat{\Sigma}^i = \sum_{j=1}^K \phi_j\phi_j^{\top}/ K$ for all $i$. Denote $\Sigma_n = \sum_{i=1}^n \EE_{(s,a)\sim \nu_i} \phi(s,a)\phi(s,a)^{\top} + \lambda I$ and $\widehat\Sigma_n = \sum_{i=1}^n \widehat{\Sigma}^i + \lambda I$ with $\lambda\in (0,1]$. Setting $K = 32 N^2 \log\left(8 N {d}/\delta\right)/\lambda^2$, with probability at least $1-\delta$, we have that for any $n\in [1,\dots, N]$,
\begin{align*}
\frac{1}{2} x^T \left({\Sigma}_n  \right)^{-1} x \leq x^T \left(\widehat{\Sigma}_n \right)^{-1} x \leq 2 x^T \left({\Sigma}_n   \right)^{-1} x,
\end{align*} for all $x$ with $\|x \|_2 \leq 1$.
\end{lemma}

\newpage
\section{Additional Experiments}
\subsection{MPPI vs TRPO}
\label{app:planners}

In this section we compare two of our practical implementation of Deep PC-MLP: a) using MPPI as the planner and use random RFF feature as $\phi$, b) using TRPO as the planner and use the fully connected layer of a random network as $\phi$. We plot the learning curves of the two in Fig.~\ref{fig:app:exp:sparse}. The settings of the experiments follow Sec.~\ref{exp:sparse}. We observe that both implementations achieve the optimal performance while all the other baselines completely fail. In terms of stability, TRPO outperforms MPPI since the performance of MPPI still oscillates before fully converges. 
\begin{figure}[h]
    \centering
    \includegraphics[width=0.4\textwidth]{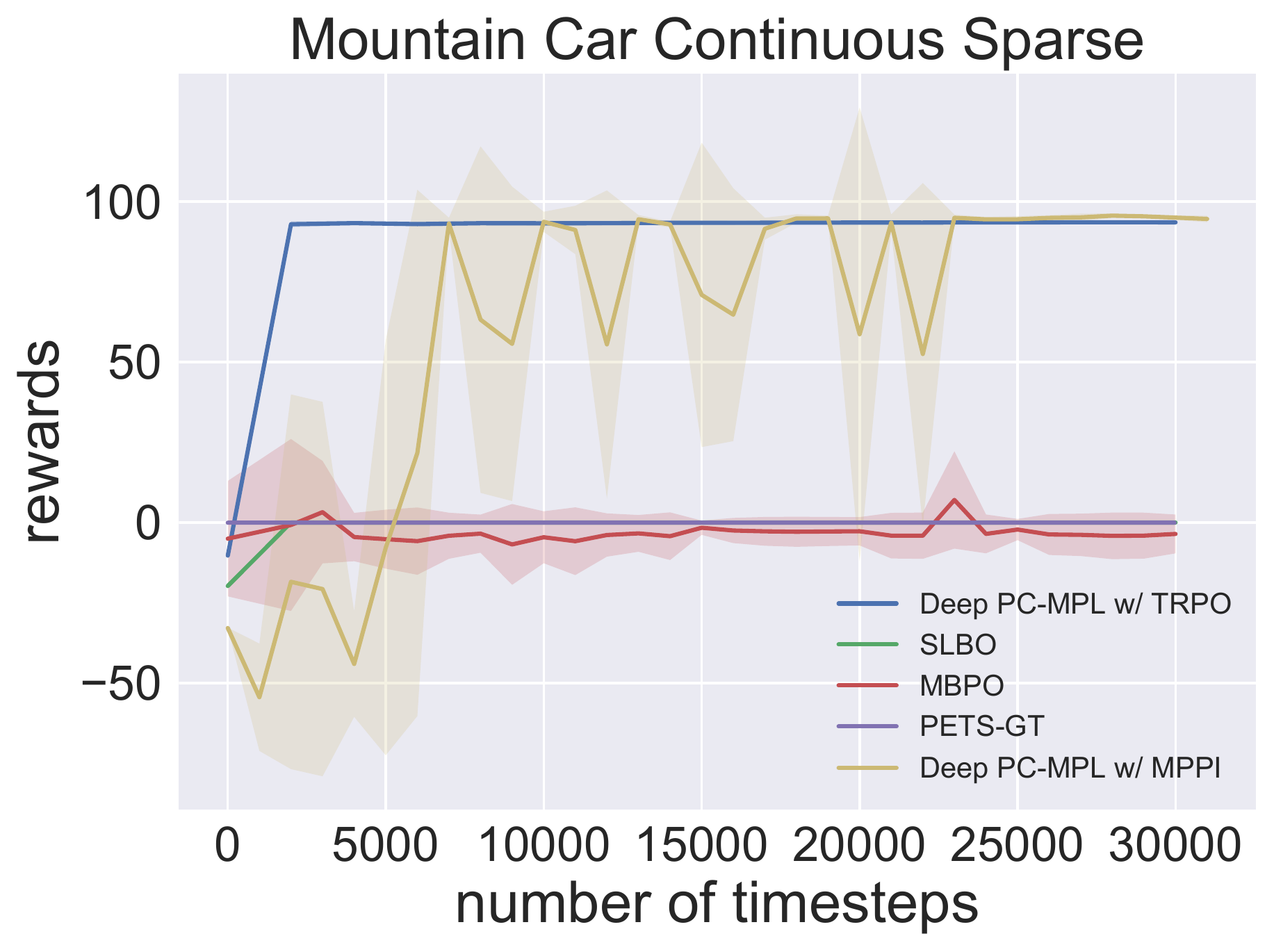}
    \centering
    \includegraphics[width=0.4\textwidth]{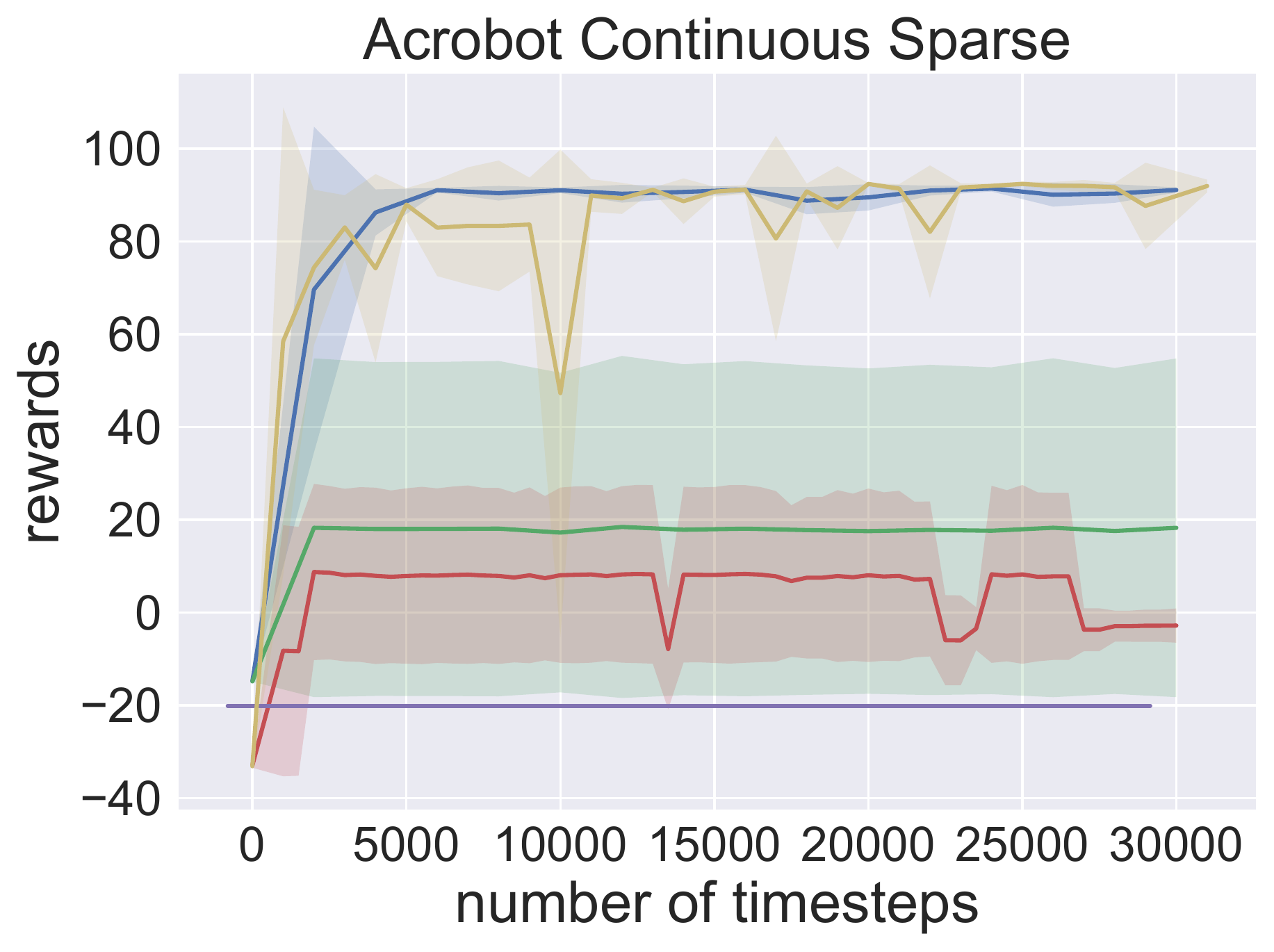}
    \caption{Performance comparison between our two practical implementations.}
    \label{fig:app:exp:sparse}
\end{figure}

\subsection{Computation Efficiency of PC-MLP}
\label{app:wall}

In this section, we investigate the wall-clock running time of Deep PC-MLP comparing with other baselines. We test on the running time of each algorithm on the MountainCar environment. We summarize the wall-clock running time in table~\ref{app:table:wall}. We see that our practical implementation can run as fast as other baselines. Taking the computation of exploration bonus into consideration, the results show that our algorithm is indeed computationally efficient.

\begin{table}[h]
    \centering
    \begin{tabular}{cc} \hline
   & Time \\ \hline
  PC-MLP & 148m43s \\
  PETS-GT & 132m50s \\ 
  MBPO & 153m30s \\
  \hline
  \end{tabular}
    \caption{Wall-clock running time comparison}
    \label{app:table:wall}
\end{table}

\section{Experimental Details}
\label{app:exp:details}

\subsection{MPPI Pseudocode}
\label{app:mppi}
Here we present the pseudocode for MPPI in Alg.~\ref{alg:mppi}:
\begin{figure}[htb]
\centering
\begin{minipage}{.8\linewidth}
\begin{algorithm}[H]
	\begin{algorithmic}[1]	
		\REQUIRE  Learned dynamics $\hat{P}$, reward $r$, number of samples $K$, shooting horizon $H$,
		noise convariance $\Sigma$, temperature $\lambda$, initial state $s_0$
		\IF{First time planning}
		\STATE Initialize $\{a_0, a_1, \dots, a_{T-1}\}$
		\ENDIF
        \FOR{$k = 1, \dots, K$}
        	   \STATE Sample $\mathcal{E}^k = \{\varepsilon_0^k,\varepsilon_1^k,\dots, \varepsilon_{T-1}^k\}$
        	   \FOR{$t = 1, \dots, T$}
        	        \STATE $S(\mathcal{E}^k) \mathrel{+}= -r(s_{t-1}, a_{t-1} + \varepsilon^k_{t-1}) + \lambda a_{t-1}^T \Sigma^{-1}\epsilon_{t-1}^k$
        	        \STATE $s_t = \hat{P}(s_{t-1}, a_{t-1} + \varepsilon^k_{t-1})$
        	   \ENDFOR
        \STATE $\beta = \min_{k}S(\mathcal{E}^k)$
        \STATE $\eta = \sum_{k=0}^{K-1} \exp{\left(-\frac{1}{\lambda} \left( S(\mathcal{E}^k) - \beta\right)\right)}$
	    \FOR{$k = 1, \dots, K$}   
	        \STATE $w(\mathcal{E}^k) = \frac{1}{\eta} \exp{\left(-\frac{1}{\lambda} \left( S(\mathcal{E}^k) - \beta\right)\right)}$
	    \ENDFOR
	    \FOR{$t = 1, \dots, T$}
	    \STATE $a_t \mathrel{+}= \sum_{k=1}^{K} w(\mathcal{E}^k)\epsilon^k_t$
	    \ENDFOR
        \ENDFOR
        \STATE $a = a_0$
        \FOR{$t = 1, \dots, T-1$}
            \STATE $a_{t-1} = a_t$
        \ENDFOR
        \STATE Initialize $a_{T-1}$
         \end{algorithmic}
	\caption{MPPI}
\label{alg:mppi}
\end{algorithm}
\end{minipage}
\end{figure}

\subsection{Hyperparameters of Deep PC-MLP}
\subsubsection{MountainCar and Acrobot}

We provide the hyperparameters we considered and finally adopted for MountainCar and Acrobot environments in Table~\ref{table:hyperparam:mountaincar:trpo} (using TRPO as the planner) and Table~\ref{table:hyperparam:mountaincar:mppi} (using MPPI as the planner).

\begin{table}[h] 
\centering
\begin{tabular}{ccc} 
\toprule
                                         & Value Considered     & Final Value \\ 
\hline
Model Learning Rate                      & \{1e-3, 5e-3, 1e-4\} & 1e-3        \\
Dynamics Model Hidden Layer Size         & \{[500,500]\}        & [500,500]   \\
Policy Learning Rate                     & \{3e-4\}             & 3e-4        \\
Policy Hidden Layer Size                 & \{[32,32]\}          & [32,32]     \\
Number of Model Updates                  & \{100\}              & 100         \\
Number of Policy Updates                 & \{40\}               & 40          \\
Number of Iterations                     & \{30,60\}            & 15          \\
Multistep Loss L                         & \{2\}                & 2           \\
Sample Size per Iteration (K)            & \{1000,2000,3000\}   & 2000        \\
Covariance Regulerazation 
Coefficient ($\lambda$)                  & \{1,0.1,0.01,0.001\} & 0.01        \\
Replay Buffer Size                       & \{30000\}            & 30000       \\
Bonus Scale ($C$)                        & \{1,2,5\}            & 5           \\
\toprule
\end{tabular}
\caption{Hyperparameter for MountainCarContinuous environment using TRPO as planner.}
\label{table:hyperparam:mountaincar:trpo}
\end{table}

\begin{table}[h] 
\centering
\begin{tabular}{ccc} 
\toprule
                                         & Value Considered     & Final Value \\ 
\hline
Learning Rate                            & \{1e-3, 5e-3, 1e-4\} & 5e-3        \\
Hidden Layer Size                        & \{64\}               & 64          \\
Number of Iterations                     & \{30,60\}            & 30          \\
Sample Size per Iteration (K)            & \{1000,2000,3000\}   & 1000        \\
Covariance Regulerazation 
Coefficient ($\lambda$)                  & \{1,0.1,0.01,0.001\} & 0.01        \\
Replay Buffer Size                       & \{10000\}            & 10000       \\
Bonus Scale ($C$)                        & \{1,2\}              & 1           \\
Dimention of RFF Feature ($|\phi|$)      & \{10,15,20,30\}      & 20          \\
MPPI Sampling Size (K)                   & \{15,100,200,300\}   & 200         \\
MPPI Shooting Horizon (H)                & \{15,30,45,60\}      & 30          \\
MPPI Temperature ($\lambda$)             & \{1,0.2,0.1,0.001\}  & 0.2         \\
MPPI Noise Covariance ($\Sigma$)         & \{0.2,0.3,0.5,1\}$I$ & 0.3$I$      \\
\toprule
\end{tabular}
\caption{Hyperparameter for MountainCarContinuous environment using MPPI as planner.}
\label{table:hyperparam:mountaincar:mppi}
\end{table}

\subsubsection{Hand Egg and Dense Reward Environments}
We provide the hyperparameters we considered and finally adopted for Hand Egg and dense reward environments in Table~\ref{table:hyperparam:dense}.
\begin{table}[h] 
\centering
\begin{tabular}{ccc} 
\toprule
                                         & Value Considered     & Final Value \\ 
\hline
Model Learning Rate                      & \{1e-3, 5e-3, 1e-4\} & 1e-3        \\
Dynamics Model Hidden Layer Size         & \{[500,500]\}        & [500,500]   \\
Policy Learning Rate                     & \{3e-4\}             & 3e-4        \\
Policy Hidden Layer Size                 & \{[32,32]\}          & [32,32]     \\
Number of Model Updates                  & \{100\}              & 100         \\
Number of Policy Updates                 & \{40\}               & 40          \\
Number of Iterations                     & \{200,100,50\}       & 50          \\
Multistep Loss L                         & \{2\}                & 2           \\
Sample Size per Iteration (K)            & \{1000,2000,4000\}   & 4000        \\
Covariance Regulerazation 
Coefficient ($\lambda$)                  & \{1,0.1,0.01,0.001\} & 0.01        \\
Replay Buffer Size                       & \{100000\}           & 100000       \\
Bonus Scale ($C$)                        & \{0.1,1\}            & 0.1           \\
\toprule
\end{tabular}
\caption{Hyperparameter for HandEgg and dense reward environment.}
\label{table:hyperparam:dense}
\end{table}

\end{document}